%% file: main.tex
\documentclass{article}

\input{header}

\PassOptionsToPackage{numbers, compress}{natbib}
\usepackage{arxiv}

\usepackage[utf8]{inputenc} % allow utf-8 input
\usepackage[T1]{fontenc}    % use 8-bit T1 fonts
\usepackage{hyperref}       % hyperlinks
\hypersetup{
    colorlinks=true,
    linkcolor={blue},
    citecolor={red},
    urlcolor={blue},
    breaklinks=true,
	plainpages=true
}
\usepackage{url}            % simple URL typesetting
\usepackage{booktabs}       % professional-quality tables
\usepackage{amsfonts}       % blackboard math symbols
\usepackage{nicefrac}       % compact symbols for 1/2, etc.
\usepackage{microtype}      % microtypography
\usepackage{xcolor}         % colors

\usepackage{enumitem}

\title{On the Computational Complexity of Self-Attention}

\author{%
   Feyza Duman Keles$^*$, Pruthuvi Mahesakya Wijewardena$^\dagger$,  Chinmay Hegde\thanks{FDK and CH are with the Tandon School of Engineering at New York University. PMW is with Microsoft. This work was supported in part by grants National Science
    Foundation (under grants CCF-2005804 and CCF-1801495) and
    USDA/NIFA (under grant 2021-67021-35329).} \\
  $^*$New York University, $^\dagger$Microsoft\\
  \texttt{\{fd2153@nyu.edu, chinmay.h\}@nyu.edu, pwijewardena@microsoft.com} \\
}

\begin{document}

\maketitle

\begin{abstract}
  Transformer architectures have led to remarkable progress in many state-of-art applications. However, despite their successes, modern transformers rely on the self-attention mechanism, whose time- and space-complexity is quadratic in the length of the input. Several approaches have been proposed to speed up self-attention mechanisms to achieve sub-quadratic running time; however, the large majority of these works are not accompanied by rigorous error guarantees. In this work, we establish lower bounds on the computational complexity of self-attention in a number of scenarios. We prove that the time complexity of self-attention is necessarily quadratic in the input length, unless the Strong Exponential Time Hypothesis (SETH) is false. This argument holds even if the attention computation is performed only approximately, and for a variety of attention mechanisms. As a complement to our lower bounds, we show that it is indeed possible to approximate dot-product self-attention using finite Taylor series in linear-time, at the cost of having an exponential dependence on the polynomial order.
\end{abstract}

\section{Introduction}

\paragraph{Motivation.} Building upon early successes in natural language processing~\cite{vaswani2017attention,kenton2019bert}, 
transformer models now form the core of virtually every state-of-the-art approach in numerous applications: computer vision and image understanding~\cite{dosovitskiy2020image,khan2021transformers}, proteomics~\cite{jumper2021highly}, code synthesis~\cite{chen2021evaluating}, and vision-language models~\cite{radford2021learning,alayrac2022flamingo}. At the heart of transformer architectures is the \emph{self-attention} mechanism, which can be viewed as a trainable ``layer'' that takes in as input a set of tokens (vectors), computes pairwise dot-products between (projected forms of) the tokens, performs a softmax operation to obtain non-negative weights, and produces output tokens using these weights. 
%This is typically followed by other (standard) dense layers, but our focus will be on the self-attention operation. 

Unfortunately, by virtue of its very definition, the standard form of self-attention requires pairwise token operations, and therefore incurs {quadratic} running time in terms of the number of input tokens. This poses serious computational challenges not just in terms of the training cost of such models, but even just for inference (forward passes) through transformer models. 
%The quadratic complexity induces further resource burdens in terms of latency and memory.
The community has long acknowledged this fact, and numerous approximate forms of self-attention that break this quadratic bottleneck have been proposed. Some approaches advocate reducing complexity through windowing, striding, or some other sparsification of the attention scores~\cite{kitaev2019reformer,zaheer2020big}. Others leverage ideas from hashing~\cite{kitaev2019reformer,choromanski2020rethinking}. Yet others propose kernelizing the attention operation~\cite{katharopoulos2020transformers,soft,chen2021skyformer}. 

Empirically, all these methods certainly seem to lead to reduced running times: in some cases, they reduce the costs from quadratic to linear. However, \emph{all} of these methods incur some form of error in the computation (compared to vanilla attention). These errors may have the undesirable (but benign) effect of drop in accuracy, or may have more dramatic effects if fed with an adversarial input. In any case, it would be useful to clearly establish rigorous guarantees on time/accuracy tradeoffs for the above methods, but these are rare. The primary theoretical question that we ask is as follows: 

\begin{quote}
\emph{What are the fundamental computational tradeoffs involved in self-attention?}
\end{quote}

\paragraph{Our contributions.} In this paper, we pursue a complementary path from most previously published work in this area. Somewhat surprisingly, we are able to establish (conditional) \emph{quadratic lower bounds} on the running time of self-attention in a large variety of settings. This quadratic barrier holds \emph{even} if we relax the self-attention operation and allow for additive, or multiplicative, errors in the computation. We also prove quadratic --- specifically, rectangular --- lower bounds even if we allow for windowing or striding. Finally, this holds even if we use kernelization (with radial basis function kernels), which to our current knowledge, achieves the Pareto time/accuracy frontier in terms of empirical performance among all fast algorithms for self-attention computation~\cite{soft}. In Table \ref{tab:summary_table}, we summarize hardness results where checkmarks indicate the proven complexities in this paper for different types of self-attention and calculation types.

Our results demonstrate that there may be a fundamental ``no free lunch'' phenomenon sitting here: it seems unlikely that we can get (provably) sub-quadratic algorithms for self-attention that are also (provably) near-accurate for all inputs.

Finally, while our primary contributions in this paper are mainly from the perspective of lower bounds, we also provide some upper bounds. Specifically, we show that a finite Taylor series approximation of the softmax function lends itself to an approximate form of self-attention that can be computed in linear time. However, a caveat of this result is that the running time now scales exponentially in the order of the Taylor polynomial.

\paragraph{Techniques.} Our proofs are rather intuitive and are based on careful reductions from the Strong Exponential Time Hypothesis (SETH). SETH-based lower bounds have attracted recent (but growing) attention from the complexity theory community, and has been used to prove hardness results for edit distance~\cite{backurs2015edit}, Frechet distance~\cite{bringmann2014walking}, dynamic programming~\cite{bringmann2015quadratic}, among many others~\cite{bringmann2021fine}. For machine learning problems, such lower bounds are less common; still, quadratic-time barriers based on SETH have been proved for kernel PCA and backpropagation through dense networks~\cite{backurs2017fine}, as well as nearest neighbors~\cite{rubinstein2018hardness}. Our results can be viewed as an addition to this body of work.

A direct reduction from SETH is cumbersome. So instead, mirroring~\cite{backurs2015edit}, we derive reductions from the Orthogonal Vectors Problem (OVP), which is known to require almost-quadratic time assuming SETH. As intermediate waypoints we visit two adaptations of OVP: the Thresholded Vectors Product Problem (TVPP), and the Bichromatic Hamming Close Pair (BHCP) problem. Reductions between these problems require construction of several ``vector gadgets", and form the bulk of the technical difficulty of our proofs. Another subtlety lies in identifying the correct temperature scaling in the softmax in order to achieve the reductions. See Section \ref{sec:hardness} for details.
 
%The same ingredients are useful in all of our proofs, indicating that any future method relying on softmax type numerics also will face the same quadratic complexity as vanilla self-attention.

\begin{table}
  \caption{Summary of our hardness results on self-attention}
  \label{tab:summary_table}
  \centering
  \begin{tabular}{lccc}
    \toprule
    & \multicolumn{3}{c}{Calculation Type}                   \\
    \cmidrule(r){2-4}
    Self-Attention & Exact & \begin{tabular}{@{}c@{}}Element-wise \\ Multiplicative Approx.\end{tabular}  & Additive Approx.   \\
    \midrule
    Exponential Dot-Product & \checkmark & \checkmark & \checkmark \\
    Softmax Dot-Product     & \checkmark & \checkmark & \xmark     \\
    Window Sliding          & \checkmark & \checkmark     & \checkmark \\
    Exponential L2-Norm     & \checkmark & \checkmark & \xmark     \\
    \bottomrule
  \end{tabular}
\end{table}

%   Transformer models that were first introduced by Vaswani et al. [] in 2017, became a very popular deep neural model because they showed remarkable progress in many natural language processing (NLP) applications. 

%   Despite their success, the time and space complexity of the Transformers are quadratic and many works are placed to make Transformers efficient. Recently, a good number of studies are dedicated to this problem,[][][][]. These efficient transformer papers tend to the self-attention mechanism complexity improvements that is one of the important components of Transformers. 

%   On the other hand, these works are deficient in theoretical guarantees. In this work, we investigate the lower bounds for the computational complexity of self-attention and approximations for the self-attention. Further, we study the theoretical guarantees for the approximations to the self-attention in worst case scenario by reduction from Strong Exponential Time Hypothesis (SETH). We prove that, if SETH is true, the time complexity of self-attention and additive, and multiplicative approximations are \emph{quadratic}. 

\section{Related work}
\label{sec:related}

\textbf{Attention mechanisms and transformers.} Ever since the seminal work of~\cite{vaswani2017attention}, transformer architectures (with self-attention layers as their primary building blocks) have become the cornerstone of state-of-the-art machine learning models. Their use cases range from large language models such as BERT~\cite{kenton2019bert} and GPT~\cite{brown2020language}, to computer vision models~\cite{dosovitskiy2020image}, \cite{liu2021swin}, \cite{khan2021transformers}, to multi-modal vision-language models~\cite{radford2021learning,alayrac2022flamingo}, to even automated code-writing~\cite{chen2021evaluating}, among many, many other examples. Therefore, a firm theoretical understanding about the statistical and computational tradeoffs involved in transformer-based models is of considerable interest. Transformers have already been shown to exhibit the universal approximation property~\cite{yun2019transformers}, but lose expressivity unless the self-attention mechanism is accompanied with skip connections~\cite{dong2021attention}. Self-attention also exhibits undesirable Lipschitz continuity properties, but this can be fixed by pursuing kernel-like alternatives~\cite{kim2021lipschitz}.

\textbf{Speeding up self-attention.} Our focus in this paper is on the running time of self-attention computation. It has been well-established that the (standard) definition of self-attention takes in as input a length-$n$ sequence of tokens of size $d$, and requires $O(dn^2)$ time to compute the output. The quadratic dependence on $n$ poses a challenge for very long input sequences, both from the training and testing perspectives. Therefore, several sub-quadratic methods for evaluating attention layers have been proposed. Approaches (such as the Reformer \cite{kitaev2019reformer}, Big Bird \cite{zaheer2020big}, Linformer \cite{wang2020linformer}, Longformer \cite{longformer}, or routing transformers~\cite{roy2021efficient}) use some combination of hashing, sparsification, or low-rank approximation to speed up the computation of the attention scores. Other approaches involve replacing the softmax-based attention with kernel approximations; cf.\ the work of \cite{katharopoulos2020transformers}, or the Nystr{\"o}mformer \cite{xiong2021nystromformer}. More recent works such as the Performer \cite{choromanski2020rethinking}, Slim \cite{likhosherstov2021sub}, or RFA \cite{peng2020random} approximate the attention computation using random projections. Methods such as SOFT \cite{soft} or the Skyformer \cite{chen2021skyformer} propose to replace softmax operations with Gaussian kernels that can then be quickly evaluated. In an orthogonal direction, \cite{bhojanapalli2021leveraging} observes that attention patterns often repeat across heads and layers, and propose reusing score computations during inference.

We found that despite the large number (and diversity) of interesting algorithmic ideas involved in the above efforts, the vast majority of these works only focus on improvement in \emph{running time}; but few (if any) theoretically characterize the \emph{error} incurred by their proposed methods. Can there ever be a method that is both fast (i.e., provably with sub-quadratic running time) as well as near-accurate (i.e., with provably small additive or multiplicative error)? Our results show that this is unlikely to be the case.

\textbf{Fine-grained complexity.}  Classical complexity theory has primarily focused on distinguishing between problems with efficient (polynomial-time) solutions versus those who don't. However, a different (and finer) picture has begun to emerge over the last decade. In particular, the focus has shifted towards precisely pinning down the \emph{exponent}, $c$, of a problem that can be solved in polynomial time $\tilde{O}(n^c)$. Many of these results are conditional, and rely on reductions from popular (but plausible) conjectures such as the Strong Exponential Time Hypothesis (SETH)~\cite{seth}, \cite{seth2}. See the relevant surveys \cite{indyk2017beyond}, \cite{rubinstein2019seth}, and \cite{bringmann2021fine} for a more concrete overview of the field. In particular, this approach has been shown to provide conditional lower bounds on well-known problems such as edit distance \cite{backurs2015edit}, Frechet distance \cite{bringmann2014walking}, dynamic time warping \cite{bringmann2015quadratic}, longest common subsequence (LCS) \cite{abboud2015tight}, Hausdorff distance \cite{bringmann2021translating}, and string matching \cite{abboud2018if}. In the context of machine learning and massive data analysis, reductions from SETH have been fruitfully applied to problems such as clustering \cite{abboud2019subquadratic}, kernel PCA \cite{backurs2017fine}, and approximate nearest neighbors \cite{rubinstein2018hardness}.

\section{Notations and Preliminaries}\label{sec:prelims}

An ordered finite set of $n$ vectors in $\R^d$ will be denoted as an $n \times d$ matrix whose rows denote the elements of the set respectively. We use upper case characters to denote both vector sets and matrices depending on the context. We use $A_i$ to denote the $i^{th}$ row of matrix $A$, or the $i^{th}$ element of ordered set $A$. We use $A_{ij}$ to denote the element at the $i^{th}$ row and $j^{th}$ column of matrix $A$.  For a positive integer $n \in \mathbb{Z}^+$, $[n]$ denotes the set of all positive integers up to $n$. 

\subsection{Background on Self-Attention}\label{sec:prelims:self-attn}

The well-established Transformer model \cite{vaswani2017attention} is based on the multi-head attention mechanism, comprising several self-attention layers running in parallel. The canonical choice of self-attention is the \emph{softmax dot-product self-attention}, defined as follows. For a given set of inputs written as $X \in \R^{n \times d}$ and trainable parameter matrices $W_q \in \R^{d \times d_q}, W_k \in \R^{d \times d_k}, W_v \in \R^{d \times d_v}$, this operation first calculates the \emph{query} ($Q=XW_q$), \emph{key} ($K=XW_k$), and \emph{value} ($V=XW_v$) matrices respectively. We assume that $d_q = d_k$. The size of $Q$ and $K$ is then $n \times d_k$, while the size of $V$ is $n \times d_v$. The softmax dot-product self-attention operation is defined as: 

\begin{equation} 
\label{eq:softmax-self-attn}
\textnormal{Attention}(Q, K, V) = \textnormal{softmax}\left(\frac{QK^T}{\sqrt{d_k}}\right)V .
\end{equation} 

Let us consider the case where $n$ is very large compared to $d_k, d_v$. By very virtue of its definition, we might expect to incur $O(n^2)$ time to compute this self-attention operation: 1) calculation of $S=QK^T/\sqrt{d_k}$ takes $O(n^2d_k)$, 2) exponentiation and calculation of row sum of $S$ takes $O(n^2)$ time, 3) division of each element of $S$ with the corresponding row sum takes $O(n^2)$, and 4) multiplication of $\softmax(QK^T)$ and $V$ takes $O(n^2d_v)$ time. Therefore, the computational complexity of this naive approach to compute self-attention scales quadratically in $n$. 

%Also, there are some other proposed self-attention mechanisms such as Gaussian kernels \cite{soft}, \cite{chen2021skyformer}. 

\paragraph{A generalized form of self-attention.}
While Eq.~\ref{eq:softmax-self-attn} is the typical way to define self-attention, we also consider a more general form. Let $f: \R^{d} \times \R^{d} \rightarrow \R$ be a function that takes two vectors in $\R^d$ as input. Then, the self-attention score matrix $S$ is defined as $S_{ij}=f(Q_i,K_j)$ for all $i,j \in [n]$. Also, let $h: \R^{n \times n}   \rightarrow  \R^{n \times n}$ be some kind of normalization function. Then a more abstract definition of self-attention can be expressed as: 
\begin{equation}\label{eq:general-self-attn}
    \textnormal{Attention}(Q, K, V) = h(S) \cdot V .
\end{equation}
In particular, for the softmax dot-product self-attention, the function $f$ is the dot-product of given vectors with normalization factor $\sqrt{d_k}$, and $h$ is the row-wise softmax function. 

In this paper, we show that there is no (provably) better algorithm than the naive $O(n^2)$ approach for calculating softmax dot-product self-attention, given hardness of SETH. We will also give similar lower bounds for various approximate forms of self-attention. Finally, we will investigate the computational complexity of generalized self-attention for more general forms of $f$.

\subsection{SETH and OVP}

Despite a remarkable amount of algorithmic effort on Boolean satisfiability (SAT) and
related problems, to date no one has invented an algorithm with faster-than-exponential ($O(2^n)$) running time; indeed, there is no polynomial-time algorithm for SAT unless $P=NP$. The Strong Exponential Time Hypothesis (SETH)~\cite{seth,seth2} can be viewed as an strengthening of this statement: for every $\epsilon$, there is no \emph{sub-exponential} ($O(2^{n(1-\epsilon)})$) time algorithm that solves SAT. 

As discussed above in Section~\ref{sec:related}, over the last decade SETH has been used to provide fine-grained lower bounds for several polynomial-time problems. Many of these results use reduction from an intermediate problem, given as follows.

\begin{defn}[Orthogonal Vectors Problem (OVP)]\label{defn:ovp}
Two sets with cardinality $A=\{a_1, \dots, a_n\}$ and $B=\{b_1, \dots, b_n\}$ are given, where $a_i,b_i \in \{0, 1\}^d$ are binary vectors for all $i \in [n]$. The problem decides if there exists at least one pair of vectors $a \in A$ and $b \in B$ such that $a^Tb = 0$.
\end{defn}

Previous work has that for all $\epsilon>0$, there is no $O(n^{2-\epsilon})$ algorithm solves OVP for $d=\omega(\log n)$ unless SETH is false \cite{WILLIAMS2005}. In other words, for any $\epsilon>0$, problem needs at least $O(n^{2-\epsilon})$ time for $d=\omega(\log n)$. In this work, we will primarily give lower bounds to the computational complexity of self-attention mechanism by showing reductions from OVP. 

\section{Hardness of Computing Self-Attention}
\label{sec:hardness}

\subsection{Adaptations of OVP}
\label{subsec:fundamentals}

Based on the definition of generalized self-attention (Eq. \ref{eq:general-self-attn}), we focus on two main forms of self-attention: (a) dot-product self-attention with $f(x, y) = e^{C x^Ty}$, (b) $\ell_2$ self-attention (or RBF kernel self-attention) with $f(x, y) = e^{-C. \norm{x-y}_2^2}$ for some temperature/scale parameter $C$ to be specified later, where $x, y$ are row vectors from matrices $Q, K$ respectively. We provide hardness guarantees for a variety of self-attention mechanisms built on top of these two types of self-attention. We achieve this by deriving reductions from two fundamental problems stated in Definitions \ref{defn:tvvp} and \ref{defn:bhcp}. 

\begin{defn}[Threshold Vectors Product Problem (TVPP)]\label{defn:tvvp}
Two sets with equal cardinality $A=\{a_1, \dots, a_n\}$ and $B=\{b_1, \dots, b_n\}$ are given, where $a_i,b_i \in \{0, 1\}^d$ are binary vectors for all $i \in [n]$. The problem decides if there exists at least one pair  $a \in A$ and $b \in B$ such that $a^Tb \geq t$ for a given $t \in [n]$.
\end{defn} 

\begin{defn}[Bichromatic Hamming Close Pair Problem  (BHCP)]\label{defn:bhcp}
Two sets with equal cardinality $A=\{a_1, \dots, a_n\}$ and $B=\{b_1, \dots, b_n\}$ are given, where $a_i,b_i \in \{0, 1\}^d$ are binary vectors for all $i \in [n]$. The problem decides if there exists at least one pair of vectors $a \in A$ and $b \in B$ such that $\norm{a-b}_2 < t$.
\end{defn}

Both TVPP and BHCP can be shown to be SETH-hard by showing reductions from OVP (Definition \ref{defn:ovp}) as shown in Lemmas \ref{lemma:ovp-to-tvpp}, \ref{lemma:ovp-to-bhfp}, and \ref{lemma:bhfp-to-bhcp}. The hardness of BHCP is an established result \cite{backurs2017fine}, but we provide an improved reduction from OVP to BHCP via a new problem: Bichromatic Hamming Far Pair problem~(BHFP). In contrast to established reductions from OVP, we get rid of additional factors of $d$ (the dimension of binary vectors) that incur during the process.

\begin{defn}[Bichromatic Hamming Far Pair Problem (BHFP)]\label{defn:bhfp} 
Two sets with equivalent cardinality $A=\{a_1, \dots, a_n\}$ and $B=\{b_1, \dots, b_n\}$ are given, where $a_i,b_i \in \{0, 1\}^d$ are binary vectors for all $i \in [n]$. The problem decides if there exists at least a pair of vectors $a \in A$ and $b \in B$ such that $\norm{a-b}_2 \geq t$.
\end{defn}

Before discussing the hardness of computing self-attention, we state the hardness guarantees of TVPP, BHFP, and BHCP formally, with proofs deferred to Appendix \ref{app:lemma-proofs}.

\begin{lemma}\label{lemma:ovp-to-tvpp}
Assume SETH. Then for any $\epsilon > 0$, the computational complexity of TVPP is  $\Omega(n^{2-\epsilon})$ for $d=\omega(\log n)$.
\end{lemma}

\begin{lemma}\label{lemma:ovp-to-bhfp}
Assume SETH. Then for any $\epsilon > 0$, the computational complexity of BHFP is  $\Omega(n^{2-\epsilon})$ for $d=\omega(\log n)$.
\end{lemma}

\begin{lemma}\label{lemma:bhfp-to-bhcp}
Assume SETH. Then for any $\epsilon > 0$, the computational complexity of BHCP is  $\Omega(n^{2-\epsilon})$ for $d=\omega(\log n)$.
\end{lemma}

\iffalse
\paragraph{Remark.}In our proofs we show the hardness of special instances of TVPP, BHFP, and BHCP where we set the value of $t$ to specific values. This also implies the hardness of these problem in general case.
\fi

% CH: we are saying this again below, so commenting it out
%Now we are ready to establish hardness guarantees of computing self-attention. The main idea in the reductions we describe is we consider a general instance of either TVPP or BHCP and convert its input into an input to computing self-attention on matrices $Q, K, V$ in linear time in $n$. We create this instance in a way such that checking the magnitude of entries in the matrix returned by self-attention function described in equation \ref{eq:general-self-attn} enables one to distinguish between true and false answers to the decision problems in TVPP or BHCP.

Below, we show a series of reductions from TVPP and BHCP to several well-studied self-attention mechanisms. We mainly modify the functions $f(.)$ and $h(.)$ (see Eq. \ref{eq:general-self-attn}) to reflect each attention mechanism in our arguments. We ignore the scaling factor $1/\sqrt{d_k}$ when computing the function $f(.)$ in dot-product self-attention for easier exposition as we only require to scale every element of $Q$ or $K$ by $1/\sqrt{d_k}$ which takes $O(nd_k)$ time. Also, we note that several approaches for efficient transformers were developed with the particular case of $Q=K$ \cite{kitaev2019reformer}, and our hardness results are valid for this specific instance where $Q=K$ and by direct reduction, and hardness guarantees hold for any universal $(Q, K)$ as well. Also, all the process is still valid for multi-head self attention which is proved in Appendix \ref{app:multihead}.

%% CH: I think this is quite standard while constructing hard instances so commenting out. can add back in if space.
%Although the input to self-attention function are $Q, K, V$ matrices, this is a mix of input sequence $X$ and trainable parameters $W_q, W_k, W_v$. To obtain the specific $Q, K, V$ we use in our reductions for any given input matrix $X$, we can define the parameter matrices in the following way: Consider $Q = XW_q$ where $Q$ is predefined according the reduction and $X$ is an arbitrary input. Now we can choose $W_q = X^+Q$ where $X^+$ is the Moore-Penrose pseudo inverse of $X$ \cite{pseudo} ( Moore-Penrose pseudo inverse exists for any matrix $X$, and it can be computed in the analytical form $X^+ = X^T(XX^T)^{-1}$ when $X$ is a full rank matrix). We can construct $W_k, W_v$ in a similar way. Showing the hardness of these special instances of input to self-attention also implies the hardness of self-attention with general input.

%the assumption  and use $Q$ in both places, however,  the hardness guarantees are universal for any $Q, K$.(We also provide a slightly different arguments without depending on the assumption $Q=K$ explicitly, see appendix\pmw{which part of appendix?}).

\subsection{Vector Gadgets for Reductions}\label{sec:vector-gadgets}

Our arguments follow by showing reductions from TVPP and BHCP instances to self-attention instances. We define a set of vector gadgets that convert input to TVPP and BHCP into inputs to self-attention functions in the following manner.

\paragraph{TVPP vector gadgets.} Let $A=\{a_1, \dots, a_n\}$ and $B=\{b_1, \dots, b_n\}$  given in TVPP, where $a_i,b_i \in \{0, 1\}^d$ are binary vectors for all $i \in [n]$. We construct our vector gadgets in the following way. First, create the matrix $Q \in \R^{2n \times d}$ with its rows as $Q_i=a_i$ for all $i \in [n]$, and $Q_{n+j}=Cb_j$ for all $j \in [n]$, where $C>0$ is a parameter which we will define in the corresponding reductions. Then we create the matrix $V \in \R^{2n \times 1}$ by setting first $n$ elements to $0$ and the second $n$ elements to $1$.

\paragraph{BHCP vector gadgets.} Let $A=\{a_1, \dots, a_n\}$ and $B=\{b_1, \dots, b_n\}$  given in BHCP, where $a_i,b_i \in \{0, 1\}^d$ are binary vectors for all $i \in [n]$. We construct our vector gadgets in the following way. First, create the matrix $Q \in \R^{2n \times d}$ with its rows as $Q_i=a_i$ for all $i \in [n]$, and $Q_{n+j}=b_j$ for all $j \in [n]$. Then we create the matrix $V \in \R^{2n \times 1}$ by setting first $n$ elements to $0$ and the second $n$ elements to $1$.

The entire construction process (including multiplying each $a_i, b_i$ by $C$) of vector gadgets takes $O(nd)$ time. We use the above gadgets in the following reductions.

\subsection{Hardness of Dot-Product Self-Attention}

We begin with hardness results of dot-product self-attention (without softmax normalization) as a warm-up. In Theorem \ref{thm:hardness-dot-product-self-attn} we show that exact self-attention, as well as element-wise multiplicative- and additive-error approximations of self-attention, all require quadratic time, conditioned on SETH. 
%The argument follows from a reduction from TVPP to computing self-attention.

\begin{theorem}\label{thm:hardness-dot-product-self-attn}
Assume SETH. For any $i, j \in [n]$, let $S_{ij} = f(Q_i, Q_j) = e^{Q_i^TQ_j}$ and for any matrix $M \in \R^{n \times n}$, let $h(M) = M$. Let $Y = h(S) \cdot V \in \R^{n \times d_v}$ be a self-attention mechanism. Provided $d_q = \omega(\log n)$, for any $\epsilon > 0$, computing a matrix $\hat{Y} \in \R^{n \times d_v}$ that satisfies any of the following conditions requires $\Omega(n^{2-\epsilon})$ time.
\begin{enumerate}[nosep,leftmargin=*]
    \item $\hat{Y} = Y$ (exact computation). 
    \item $|\hat{Y}_{ij} - Y_{ij}| \le \mu |Y_{ij}|$ for all $i \in [n]$ and $j \in [d_v]$ where $0 \le \mu < 1 $~(multiplicative approximation).
    \item $|\hat{Y}_{ij} - Y_{ij}| \le \mu$ for all $i \in [n]$ and $j \in [d_v]$ where $0 \le \mu$~(additive approximation).
\end{enumerate}
\end{theorem}

\begin{proof}

%Consider two sets $A=\{a_1, \dots, a_n\}$ and $B=\{b_1, \dots, b_n\}$  given in TVPP, where $a_i,b_i \in \{0, 1\}^d$ are binary vectors for all $i \in [n]$. The problem is to decide if there exists at least a pair  $a \in A$ and $b \in B$ that satisfies $a^Tb \geq t$ for a given $t \in [n]$.

Suppose that the matrices $Q$ and $V$ are constructed as TVPP vector gadgets described above in Section \ref{sec:vector-gadgets}. With this, we have $Y =  h(S) \cdot V \in \R^{n \times 1}$. Since $Y$ is a vector, we slightly abuse notation and define $Y_i$ as the $i^{th}$ element of $Y$. Now consider the first $n$ elements of $Y$. Since $Y = SV$, we have that 
\[ Y_i = \sum_{j=1}^n e^{Ca_i^Tb_j} \text{ for any } i \in [n]. \]

Now we check the magnitude of $Y_i, i \in [n]$ in order to distinguish between true and false cases in TVPP. It takes $O(n)$ time to check each $Y_i, i \in [n]$. First we focus on the exact computation. Consider the following two cases.

\textbf{Case 1.} There are no pairs $a \in A$ and $b \in B$ with $a^Tb \ge t$, that is for all $i, j \in [n]$, we have $a_i^Tb_j \le t-1$, and $e^{C a_i^Tb_j} \le e^{C(t-1)} $. Then for all $l \in [n]$, $Y_l \le n e^{C(t-1)} := \delta$.

\textbf{Case 2.} There is a pair $a \in A$ and $b \in B$ with $a^Tb \ge t$, that is for some $i, j \in [n]$, we have $a_{i}^Tb_{j} \ge t$, and $e^{C a_{i}^Tb_{j}} \ge e^{Ct}$. Thus for some $l \in [n]$, we have $Y_l \ge e^{Ct} :=\Delta$ 

To distinguish between the two cases, it is sufficient to have $\Delta > \delta$. But this holds when $C = 2 \log n$. 

Now let us consider multiplicative approximation error. With a $\mu$-multiplicative factor, if there are no pairs $a \in A$ and $b \in B$ with $a^Tb \ge t$, then we have for all $l \in [n], \hat{Y}_l \le (1+\mu) Y_l \le (1+\mu) n e^{C(t-1)} := \hat{\delta}$. On the other hand, if there is a pair $a \in A$ and $b \in B$ with $a^Tb \ge t$, then for some $l \in [n]$, we have $\hat{Y}_l \ge (1-\mu) Y_l \ge (1-\mu) e^{Ct} := \hat{\Delta}$. In order to distinguish between two cases, it is sufficient to have $\hat{\Delta} > \hat{\delta}$ and this inequality holds with $C = 2 \log (\frac{1+\mu}{1-\mu}n)$.

Finally we look at additive approximation error. With a $\mu$-additive factor, if there are no pairs $a \in A$ and $b \in B$ with $a^Tb \ge t$, then we have for all $l \in [n], \hat{Y}_l \le Y_l + \mu \le  n e^{C(t-1)} + \mu := \hat{\delta}$. On the other hand, if there is a pair $a \in A$ and $b \in B$ with $a^Tb \ge t$, then for some $l \in [n]$, we have $\hat{Y}_l \ge  Y_l -\mu \ge e^{Ct} - \mu := \hat{\Delta}$. In order to distinguish between two cases, it is sufficient to have $\hat{\Delta} > \hat{\delta}$ and this inequality holds with $C = 2 \log (n + 2 \mu)$.

Thus, if there is an algorithm for computing self-attention up to an element-wise multiplicative or additive error $\mu$ that runs in $O(n^{2-\epsilon})$ time, this entire process takes at most $O(n + nd + n^{2-\epsilon}) = O(n^{2-\epsilon})$ as long as $d = o(n^{1-\epsilon})$. Therefore, this algorithm decides if there exists at least a pair of vectors $a \in A$ and $b \in B$ such that $a^Tb \ge t$ in $O(n^{2-\epsilon})$ time, which contradicts the hardness result of TVPP (Lemma \ref{lemma:ovp-to-tvpp}). This completes the proof.
\end{proof}

The above proof is for computing the output of self-attention. As an easy consequence, we can also show that computing the self-attention \emph{score matrix}, $S$, with either exact or element-wise multiplicative/additive error requires quadratic time, conditioned on SETH. The argument follows a similar proof as Theorem \ref{thm:hardness-dot-product-self-attn}, and is given in Appendix \ref{S_Approximations} in detail.

\subsection{Hardness of Softmax Dot-Product Self-Attention}

Now we establish hardness guarantees for computing standard softmax dot-product self-attention. The difference from vanilla self-attention is that the function $h(.)$ now normalizes input rows. Discussion of additive approximation for this part is given in Appendix \ref{app:Additive}.

\begin{theorem}\label{thm:hardness-softmax-self-attn}
Assume SETH. For any $i, j \in [n]$, let $S_{ij} = f(Q_i, Q_j) = e^{Q_i^TQ_j}$ and for any matrix $M \in \R^{n \times n}$, let $h(M) \in \R^{n \times n}$ be the matrix where for all $i, j \in [n]$, $\{ h(M) \}_{ij} = \frac{M_{ij}}{\sum_{k=1}^n M_{ik}}$. Let $Y = h(S) \cdot V \in \R^{n \times d_v}$ be a self-attention. Then provided $d_q = \omega(\log n)$, for any $\epsilon > 0$, computing a matrix $\hat{Y} \in \R^{n \times d_v}$ that satisfies any of the following conditions requires $\Omega(n^{2-\epsilon})$ time.
\begin{enumerate}[nosep,leftmargin=*]
    \item $\hat{Y} = Y$(exact).
    \item $|\hat{Y}_{ij} - Y_{ij}| \le \mu |Y_{ij}|$ for all $i \in [n]$ and $j \in [d_v]$ where $0 \le \mu < 1 $~(multiplicative approximation).
\end{enumerate}
\end{theorem}

\begin{proof}

%Consider two sets $A=\{a_1, \dots, a_n\}$ and $B=\{b_1, \dots, b_n\}$  given in TVPP, where $a_i,b_i \in \{0, 1\}^d$ are binary vectors for all $i \in [n]$. The problem is to decide if there exists at least a pair  $a \in A$ and $b \in B$ that satisfies $a^Tb \geq t$ for a given $t \in [n]$.

The proof is technically similar to the one above. Suppose that the matrices $Q$ and $V$ are constructed according to TVPP vector gadgets described in Section \ref{sec:vector-gadgets}. With this we have $Y =  h(S) \cdot V \in \R^{n \times 1}$. Consider the first $n$ elements of $Y$. Since $h(.)$ act as the row-wise softmax function, we have

\[ 
Y_i = \sum_{j=n+1}^{2n} S_{ij} = \sum_{j=1}^{n} \frac{e^{Ca_i^Tb_j}}{\sum_{k=1}^{n}e^{a_i^Ta_k}+\sum_{k=1}^{n}e^{Ca_i^Tb_k}} =  
\frac{\sum_{j=1}^n e^{Ca_i^Tb_j}}{\sum_{j=1}^{n}e^{a_i^Ta_j}+\sum_{j=1}^{n}e^{Ca_i^Tb_j}}. 
\]

Again, first we focus on exact computation and consider two cases.

\textbf{Case 1.} There are no pairs $a \in A$ and $b \in B$ with $a^Tb \ge t$, that is for all $i, j \in [n]$, we have $a_i^Tb_j \le t-1$, and $e^{C a_i^Tb_j} \le e^{C(t-1)} $. For a function $\frac{x}{x+y}$, the maximum value is achieved at maximum $x$ and minimum $y$ values. Thus, for all $l \in [n]$, $Y_l \le  \frac{n e^{C(t-1)}}{n e^{C(t-1)}+n} = \frac{e^{C(t-1)}}{e^{C(t-1)}+1}:=\delta$.

\textbf{Case 2.} There is a pair $a \in A$ and $b \in B$ with $a^Tb \ge t$, that is for some $i, j \in [n]$, we have $a_{i}^Tb_{j} \ge t$. Then the row sum corresponding to that $i, j$ pair is $\sum_{j=1}^{n} e^{Ca_i^Tb_j} \geq e^{Ct} + (n-1)e^0 = e^{Ct} + (n-1)$. For a function $\frac{x}{x+y}$, the minimum value is achieved at minimum $x$ and maximum $y$ values. Thus, for some $l \in [n]$, we have $Y_l \ge \frac{e^{Ct} + (n-1)}{e^{Ct} + (n-1) + n e^d}:=\Delta $
and $e^{C a_{i}^Tb_{j}} \ge e^{Ct}$. 

In order to distinguish between two cases, it is sufficient to have $\Delta > \delta$ which means we require $[e^{Ct} + (n-1)] > n e^d [e^{C(t-1)] }$. This holds with $C = \log n + d$.

Next, consider multiplicative error approximation. Select same TVPP vector gadgets except this time the matrix $V \in \R^{2n \times 1}$ is set first $n$ elements to $1$ and the second $n$ elements to $0$. Since $|\hat{Y}_l - Y_l| \le \mu |Y_l|$ for all $i \in [2n]$, we have $ (1-\mu) Y_i \le \hat{Y}_i \le (1+\mu)Y_i$. Now consider the values of $\hat{Y}_i, i \in [n]$ in the following two cases.

\textbf{Case 1.} There are no pairs $a \in A$ and $b \in B$ with $a^Tb \ge t$, that is for all $i, j \in [n]$, we have $a_i^Tb_j \le t-1$. This means that $\sum_{j=1}^{n}e^{Ca_i^Tb_j} \leq n e^{C(t-1)}$. For a function $\frac{x}{x+y}$, the minimum value is achieved at the minimum $x$ and maximum $y$ values. Thus, for all $l \in [n]$, $Y_l \ge  \frac{n }{n+n e^{C(t-1)}} = \frac{1}{e^{C(t-1)}+1}$ which means that for all $l \in [n]$, we have $\hat{Y_l} \ge (1-\mu)Y_l \ge  (1-\mu)\frac{1}{e^{C(t-1)}+1}:=\Delta$.

\textbf{Case 2.} There is a pair $a \in A$ and $b \in B$ with $a^Tb \ge t$, that is for some $i, j \in [n]$, we have $a_{i}^Tb_{j} \ge t$.  Then the row sum corresponding to that $i, j$ pair is $\sum_{j=1}^{n} e^{Ca_i^Tb_j} \geq e^{Ct} + (n-1)e^0 = e^{Ct} + (n-1)$. For a function $\frac{x}{x+y}$, the maximum is achieved at the maximum $x$ and minimum $y$ values. Thus, for some $l \in [n]$, we have $Y_l \le \frac{n e^d}{e^{Ct} + (n-1) + n e^d}$ which means that for some $l \in [n]$, we have $\hat{Y_l} \le (1 + \mu) Y_l \le (1+\mu) \frac{n e^d}{e^{Ct} + (n-1) + n e^d}:=\delta.$

In order to distinguish between the two cases, it is sufficient to have $\Delta>\delta$. This holds with $C = \log(\frac{2(1+\mu)}{1-\mu} n) + d$. Thus, if there is an algorithm for computing self-attention up to an element-wise multiplicative error $\mu$ that runs in $O(n^{2-\epsilon})$ time, this entire process takes at most $O(n + nd + n^{2-\epsilon}) = O(n^{2-\epsilon})$ as long as $d = o(n^{1-\epsilon})$, and this algorithm decides if there exists a pair of vectors $a \in A, b \in B$ such that $a^Tb \ge t$ in $O(n^{2-\epsilon})$ time, contradicting TVPP hardness.
\end{proof}

\paragraph{Remark.} For the multiplicative error approximations, we can set $\mu$ to a value arbitrarily close to $1$ with a dependence on $n$, i.e. $\mu = 1 - \Theta(1/n^x)$ for a sufficiently large(constant) order $x$. For example, with $x = 2$, we require $C > 2 \log(\frac{2 - \Theta(1/n^2)}{\Theta(1/n^2)}n) \approx 2[\log 2 + 3 \log(\Theta(n))]$. Therefore, to distinguish between two cases with a $\mu$ multiplicative error approximation with $\mu = 1 - \Theta(1/n^2)$, we only require $C = O(\log(n))$. We discuss this matter in more detail in the Appendix \ref{app:Remark}. 

\subsection{Hardness of Sliding Window Dot-Product Self-Attention}
We now consider well-known less-expensive alternatives to standard self-attention. A popular  example is sliding window self-attention \cite{longformer}. Here, we evaluate an element of the score matrix $S_{ij}$ only if the difference between $i$ and $j$ is within a fixed window size $w$; else we set it to zero. This reduces the running time to $O(nw)$, which can be small if the window size is small. However, we show that such a \emph{rectangular} complexity is unavoidable.

\begin{theorem}\label{thm:sliding-window-exp-dot-product}
Assume SETH. Let $f: \{\R^d, \R^d\} \longrightarrow \R$ as $f(x, y) = e^{ x^Ty}$. For set $Q$ of vectors $Q_1, \dots, Q_n$, we define the matrix $S \in \R^{n \times n}$ as
\begin{equation*}
 (S)_{ij}=\begin{cases} 
    f(Q_i,Q_j) & |i-j| \le w/2  \\
    0 & \text{otherwise} 
    \end{cases}
\end{equation*}
Also for any matrix $M \in \R^{n \times n}$, let $h(M)=M$. Let $Y=h(S)\cdot V \in \R^{n \times d_v}$ be a self-attention.
Then for any $\epsilon > 0$, computing a matrix $\hat{Y} \in \R^{n \times d_v}$ that satisfies any of the following conditions requires $\Omega(nw^{1-\epsilon})$ time when $d_q = \omega(\log w)$ and $w=\omega(d_q)$.
\begin{enumerate} [nosep,leftmargin=*]
    \item $\hat{Y} = Y$(exact).
    \item $|\hat{Y}_{ij} - Y_{ij}| \le \mu |Y_{ij}|$ for all $i \in [n]$ and $j \in [d_v]$ where $0 \le \mu < 1 $~(multiplicative approximation).
    \item $|\hat{Y}_{ij} - Y_{ij}| \le \mu$ for all $i \in [n]$ and $j \in [d_v]$ where $0 \le \mu$~(additive approximation).
\end{enumerate}
\end{theorem}

\begin{proof}
See Appendix \ref{app:sliding-window-hardness} for a full proof. First, a TVPP problem with sets size $k=\sqrt{nw}$ are constructed. Then sliding-window self attention is computed for all the points in order, so that for each TVPP problem, all the pairs are in the window. An appropriate value matrix $V$ is constructed for selection of this pairs. If the overall self attention can be calculated in $O((nw)^{1-\epsilon})$ time, then TVPP problem can be solved in $O(k^{2-\epsilon})$ which contradicts Lemma \ref{lemma:ovp-to-tvpp}. 
\end{proof}

\subsection{Hardness of \texorpdfstring{$\ell_2 \textnormal{-} $}{} Self-Attention}

We now establish hardness guarantees for computing $\ell_2$-self-attention, which replaces a softmax with an RBF kernel and is the core idea underlying both SOFT~\cite{soft} and Skyformer~\cite{chen2021skyformer}, the current state-of-the-art in fast self-attention operations. %In Theorem \ref{thm:hardness-l2-self-attn} we show that computing both exact self-attention and element-wise multiplicative error approximation of self-attention requires quadratic time in $n$ conditioned on SETH. Our argument shows that the result hold as long as the the parameter $C > O(\log n)$ (corresponding temperature parameter in RBF kernel being $1/C$).
In our argument, we adopt a similar proof technique employed in~\cite{backurs2017fine}, who establish quadratic hardness of kernel PCA assuming SETH. However, our proof involves a different chain of reductions: OVP $\rightarrow$ BHFP $\rightarrow$ BHCP $\rightarrow$ kernel computation. Discussion of additivite approximation for this part is given in Appendix \ref{app:Additive}.

\begin{theorem}\label{thm:hardness-l2-self-attn}
Assume SETH. For any $i, j \in [n]$, let $S_{ij} = f(Q_i, Q_j) = e^{C. \norm{Q_i-Q_j}_2^2}$ where $C$ is a parameter and for any matrix $M \in \R^{n \times n}$, let $h(M) = M$. Let $Y = h(S).V \in \R^{n \times d_v}$ be a self-attention. Then for any $\epsilon > 0$, computing a matrix $\hat{Y} \in \R^{n \times d_v}$ that satisfies any of the following conditions requires $\Omega(n^{2-\epsilon})$ time when $d_q = \omega(\log n)$.
\begin{enumerate} [nosep,leftmargin=*]
    \item $\hat{Y} = Y$(exact).
    \item $|\hat{Y}_{ij} - Y_{ij}| \le \mu |Y_{ij}|$ for all $i \in [n]$ and $j \in [d_v]$ where $0 \le \mu < 1 $ (multiplicative approximation).
\end{enumerate}
\end{theorem}

\begin{proof} (Sketch).
The $Q$ and $V$ matrices and constructed according to BHCP vector gadgets described in Section \ref{sec:vector-gadgets}. With these, we define $Y =  h(S) \cdot V \in \R^{n \times 1}$. Considering the first $n$ elements of $Y$ and a suitable selection of $C$, our goal is to make the two cases (whether there is a solution for BHCP or not) distinguishable. Thus, if there is an algorithm for computing self-attention up to an element-wise multiplicative error $\mu$ that runs in $O(n^{2-\epsilon})$ time, this algorithm decides if there exists a solution or not for BHCP in $O(n^{2-\epsilon})$ time, which contradicts BHCP hardness (Lemma \ref{lemma:bhfp-to-bhcp}). %For detailed proof, consult Appendix \ref{app:proof_l2_self_attention}.
\end{proof}

\section{Polynomial Approximations of Self-Attention}

We have shown (conditional) hardness results for a number of well-known self-attention mechanisms. We conclude with some upper bounds. Given the query $Q$, key $K$, and value $V$ matrices, we show that when $f(Q_i, K_j)$ is a polynomial of order $p$ ($p$ is an integer $\ge$ 0) of $Q_i^TK_j$ and $h(.)$ is either the identity function or row-wise normalization, one can compute self-attention in linear time. However, now the time complexity scales exponentially with $p$. As a special case of this, we show that one can approximate dot-product softmax self-attention in linear time in $n$ by using finite Taylor series approximation: % where $x$ represents $Q_i^TK_j$ in Corollary \ref{corr:poly-approx-dot-product-softmax-self-attn}. 
%\begin{equation}\label{eq:exponent-approx}
$
    e^x \approx \sum_{k=0}^p \frac{x^k}{k!}.
$
%\end{equation}
In what follows we use the fact that $d_q = d_k$ and use $d_q$ in both places. Recall the dimensions of matrices $Q \in \R^{n \times d_q}, K \in \R^{n \times d_q}, V \in \R^{n \times d_v}$.

Lemma \ref{lem:pth-order-poly} shows that the product $SV$(without the row-wise normalization) can be computed in linear time $n$. Lemma \ref{lem:row-wise-normalization} shows that the denominator for row-wise normalization can be computed in linear time in $n$. Once we have the denominator for each row, $h(S) \cdot V$ can be computed by dividing each row of $SV$ by the corresponding denominator.

\begin{lemma}\label{lem:pth-order-poly}
Let $p$ be an integer $\ge 0$ and let $S_{ij}$ be $C \cdot (Q_i^TK_j)^p$ where $i, j \in [n]$ and $C$ is a constant. Then $SV$ can be computed in $O(nd_q^pd_v)$ time.
\end{lemma}

\begin{lemma}\label{lem:row-wise-normalization}
Let $p$ be an integer $\ge 0$. Then for all $i \in [n]$, $\sum_{\hat{j}=1}^n C \cdot (Q_i^TK_{\hat{j}})^p, \hat{j} \in [n]$ where $C$ is a constant can be computed in $O(nd_q^p)$ time.
\end{lemma}

\textit{Proof sketch of Lemma \ref{lem:pth-order-poly} and \ref{lem:row-wise-normalization}.} When $f(Q_i, V_j) = Q_i^TV_j$, we directly multiply matrices $K^T$ and $V$ to obtain $K^TV$ in $O(d_qnd_v)$ time, then $Q$ and $K^TV$ in $O(nd_qd_v)$ gives the desired matrix $SV$ in $O(n)$ time. The sum of a row of $S$ can be simply computed by storing the sum of vectors $K_j, j \in [n]$ in memory in $O(nd_q)$ time and reusing this for each row $i$ to compute $\sum_{j=1}^n Q_i^TK_j = Q_i^T \sum_{j=1}^n K_j$ in $O(nd_q)$ time. This idea can be extended to $(Q_i^TK_j)^p$. \qed

\begin{theorem}\label{thm:poly-approx-self-attn}
Let $p$ be an integer $\ge 0$. If $S_{ij} = f(Q_i, K_j)$ is a polynomial function of order $p$ of $Q_i^TK_j$ and $h(.)$ performs row-wise normalization, then $h(S) \cdot V$ can be computed in $O(nd_q^pd_v)$ time.
\end{theorem}

\begin{proof}
The result is implied by Lemma \ref{lem:pth-order-poly} and Lemma \ref{lem:row-wise-normalization}. Let $x_{ij} = Q_i^TK_j; i, j \in [n]$. Define the polynomial function of order $p$ as $\sum_{z=0}^p c_z x_{ij}^z$ where $c_z, z \in \{ 0, \dots, p \}$ are constants. Let $S^{(1)}, \dots, S^{(p)} \in \R^{n \times n}$ where $S_{ij}^{(z)} = c_z x_{ij}^z$. Now we can write $S_{ij} = \sum_{z=0}^p S_{ij}^{(z)}$, thus $SV = \sum_{z=0}^p S_{ij}^{(z)}V$. By Lemma \ref{lem:pth-order-poly}, each term of this summation can be computed in $O(nd_q^zd_v)$ time. Therefore the overall time complexity of computing $SV$ is $O(nd_q^pd_v)$~(considering the largest exponent of $d_q$ is when $z=p$).

Let $s_i$ be the sum of the elements of $i^{th}$ row of $S$. What remains is computing $h(S). V$. This can be computed indirectly by first computing $SV$ and the dividing each row $i$ of $SV$ by $s_i$. If we have precomputed each $s_i, i \in [n]$, then this process takes $O(nd_v)$ time, since we are dividing each element of $SV ( \in \R^{n \times d_v})$ by a scaler. What remains to show is that computing $s_i, \forall i \in [n]$ takes $O(npd_q^p)$ time. Observe that 
\[s_i = \sum_{\hat{j}=1}^n S_{ij} = \sum_{\hat{j}=1}^n \sum_{z=0}^p S_{i\hat{j}}^{(z)} = \sum_{z=0}^p \sum_{\hat{j}=1}^n  S_{i\hat{j}}^{(z)} = \sum_{z=0}^p \sum_{\hat{j}=1}^n c_z x_{i\hat{j}}^z \]

From Lemma \ref{lem:row-wise-normalization}, each term of the outer summation over $z$ can be computed $O(nd_q^z)$ time for all $i \in [n]$ and overall time complexity is $O(nd_q^p)$(taking the largest exponent of $d_q$ similarly). 
\end{proof}

\begin{corollary}\label{corr:poly-approx-dot-product-softmax-self-attn}
Let $p$ be an non-negative integer. The $p^{th}$ order polynomial approximation of dot-product softmax self-attention using matrices $Q, K, V$ with finite Taylor series can be computed in $O(nd_q^pd_v)$ time.

\end{corollary}

\begin{proof}
The result follows from replacing constants $c_z = \frac{1}{z!}$ for $z = 0, \dots, p$ in Theorem \ref{thm:poly-approx-self-attn}.
\end{proof}

\section{Conclusions}\label{sec:conclusions}

In this paper we investigate fundamental bounds on the computational complexity of self-attention. We examine numerous state-of-the-art self-attention mechanisms, and prove quadratic (or rectangular) lower bounds assuming the Strong Exponential Time Hypothesis (SETH). Even though a large number of recent works have proposed fast approximations to self-attention, our results imply that it may be difficult to both overcome the quadratic runtime barrier while still retaining high accuracy. On the positive side, we show that linear-time computation is possible if we choose the score computation function in the form of a polynomial.

Our work leaves open several directions. At a high level our theorems establish a result between `exponential' and `polynomial' forms of self-attention, but having a clearer picture of the landscape may be helpful. Moreover, our results are for worst-case inputs; similar hardness results on average-case inputs is an interesting direction. Finally, we leave the door open for the possibility of randomized algorithms that achieve sub-quadratic complexity and are correct with high probability.  
%\pmw{what property separates hard self-attention from others: exponential family hard, polynomial not hard?... this show worst case guarantees, what about average case? or the special cases where the data distribution may help? Any properties of $W_q, W_k, W_v$, loss functions and regularized?.., more accurate polynomial approximation?} \fd{Also there maybe an algorithm that the complexity is linear(or sub-quadratic) with high probability}

%\section*{Limitations}

\bibliographystyle{IEEEtran}
\bibliography{refs}

\appendix

\section*{Appendix}

\section{Proofs for Hardness of TVPP, BHCP and BHFP}
\label{app:lemma-proofs}

\paragraph{Lemma 1.}\textit{Assume SETH. Then for any $\epsilon > 0$, the computational complexity of TVPP is  $\Omega(n^{2-\epsilon})$ for $d=\omega(\log n)$.}

\begin{proof}
Consider the two sets $A=\{a_1, \dots, a_n\}$ and $B=\{b_1, \dots, b_n\}$ given in OVP, where $a_i,b_i \in \{0, 1\}^d$ are binary vectors for all $i \in [n]$. For any vector $a_i \in A$, let us define $\bar{a_i} \in \R^{2d}$ as the concatenation of vector $a_i$ and vector $(\Vec{1}-a_i)$, where $\Vec{1}:=(1,1,\dots,1) \in \R^d$. Then we have $\|\bar{a_i}\|_1 = \|{a_i}\|_1 + \|\Vec{1}-{a_i}\|_1 = d$. Now define $\bar{A}$ as the set of $\bar{a_i}$s. For any vector $b_j \in B$, let us define $\bar{b_j} \in \R^{2d}$ as the concatenation of vector $(\Vec{1}-b_j)$ and vector $\Vec{1}$. Now define $\bar{B}$ as the set of $\bar{b_j}$. Because the overall dimensions of $A, B$ and $\bar{A}, \bar{B}$ are $nd$ and $2nd$ respectively, this process takes $O(nd)$ time. Now for any $i,j \in [n]$ we have

\[ \bar{a_i}^T \bar{b_j} = {a_i}^T(\Vec{1}-b_j) + (\Vec{1}-{a_i})^T \Vec{1} = \|a_i\|_1-a_i^Tb_j+\|\Vec{1}-{a_i}\|_1 = d - a_i^Tb_j. \]

$\bar{a_i}^T \bar{b_j} = d$ if and only if ${a_i}^Tb_j=0$. Now if we run the algorithm for TVPP with the threshold $t=d$ on the sets $\bar{A}$ and $\bar{B}$ to find a pair $\bar{a} \in \bar{A}$ and $\bar{b} \in \bar{B}$ that satisfies $\bar{a}^T\bar{b} \ge d$, then we can conclude that there is a pair $a \in A$ and $b \in B$ that satisfies $a^T{b} = 0$.

Thus, if there is an algorithm for TVPP that runs in $O(n^{2-\epsilon})$ time, this entire process takes at most $O(nd + n^{2-\epsilon}) = O(n^{2-\epsilon})$ as long as $d = o(n^{1-\epsilon})$, and this algorithm decides if there exists at least a pair of vectors $a \in A$ and $b \in B$ such that $a^Tb = 0$ in $O(n^{2-\epsilon})$ time, which contradicts the hardness result of OVP. As a result, for at least one $t$, TVPP problem cannot be solved in $O(n^{2-\epsilon})$ time. (We note that for $t=1$, this problem have a linear time algorithm. Selecting rows of matrix $Q$ as elements of set $A$, rows of matrix $K$ as elements of set $B$, and rows of matrix $V$ as 1, then calculating $Y=QK^TV$ takes linear time on $n$, by firstly calculating $K^TV$, then calculating $Q(K^TV)$. If there is any positive value of $Y$, then we can conclude that there is a pair $a \in A$ and $b \in B$ that satisfies $a^T{b} \ge 1$.)

\end{proof}

\paragraph{Lemma 2.}\textit{Assume SETH. Then for any $\epsilon > 0$, the computational complexity of BHFP is  $\Omega(n^{2-\epsilon})$ for $d=\omega(\log n)$.}

\begin{proof}
Consider the two sets $A=\{a_1, \dots, a_n\}$ and $B=\{b_1, \dots, b_n\}$ given in OVP, where $a_i,b_i \in \{0, 1\}^d$ are binary vectors for all $i \in [n]$. 

For any vector $a_i \in A$, let us define $\bar{a_i} \in \R^{3d}$ as the concatenation of vector $a_i$, vector $(\Vec{1}-a_i)$ and vector $\Vec{0}$ where $\Vec{1}:=(1,1,\dots,1) \in \R^d$ and $\Vec{0}:=(0,0,\dots,0) \in \R^d$. Then we have $\bar{a_i}^T\bar{a_i} = {a_i}^Ta_i + (\Vec{1}-{a_i})^T(\Vec{1}-{a_i}) = d$. Now define $\bar{A}$ as the set of $\bar{a_i}$s. For any vector $b_j \in B$, let us define $\bar{b_j} \in \R^{3d}$ as the concatenation of vector $b_j$, vector $\Vec{0}$, and vector $(\Vec{1}-b_j)$. Then we have $\bar{b_j}^T\bar{b_j} = {b_j}^Tb_j + (\Vec{1}-{b_j})^T(\Vec{1}-{b_j}) = d$. Now define $\bar{B}$ as the set of $\bar{b_j}$. Because the overall dimensions of $A, B$ and $\bar{A}, \bar{B}$ are $nd$ and $3nd$ respectively, this process takes $O(nd)$ time. Now for any $i,j \in [n]$ we have

\[ \bar{a_i}^T \bar{b_j} = {a_i}^T{b_j} + (\Vec{1}-{a_i})^T \Vec{0} + \Vec{0}^T (\Vec{1}-b_j) = {a_i}^T b_j. \] 

The squared $\ell_2$ distance between $\bar{a_i}$ and $\bar{b_j}$ for any $i,j \in [n]$ is
\[ \|\bar{a_i}-\bar{b_j}\|_2^2=(\bar{a_i}-\bar{b_j})^T(\bar{a_i}-\bar{b_j}) = \bar{a_i}^T\bar{a_i} + \bar{b_j}^T\bar{b_j} - 2 \bar{a_i}^T\bar{b_j} = d + d - 2 \bar{a_i}^T\bar{b_j}  = 2d - 2{a_i}^Tb_j. \]

$\|\bar{a_i}-\bar{b_j}\|_2^2 = 2d$ if and only if ${a_i}^Tb_j=0$. Now if we run the algorithm for BHFP with the threshold $t=\sqrt{2d}$ on the sets $\bar{A}$ and $\bar{B}$ to find a pair $\bar{a} \in \bar{A}$ and $\bar{b} \in \bar{B}$ that satisfy $\|\bar{a}-\bar{b}\|_2 \ge \sqrt{2d}$, then we can conclude that there is a pair $a \in A$ and $b \in B$ that satisfies $a^T{b} = 0$.

Thus, if there is an algorithm for BHFP that runs in $O(n^{2-\epsilon})$ time, this entire process takes at most $O(nd + n^{2-\epsilon}) = O(n^{2-\epsilon})$ as long as $d = o(n^{1-\epsilon})$, and this algorithm decides if there exists at least a pair of vectors $a \in A$ and $b \in B$ such that $a^Tb = 0$ in $O(n^{2-\epsilon})$ time, which contradicts the hardness result of OVP. As a result, for at least one $t$, BHFP problem cannot be solved in $O(n^{2-\epsilon})$ time.

\end{proof}

\paragraph{Lemma 3.}\textit{Assume SETH. Then for any $\epsilon > 0$, the computational complexity of BHCP is  $\Omega(n^{2-\epsilon})$ for $d=\omega(\log n)$.}

\begin{proof}
Consider the two sets $A=\{a_1, \dots, a_n\}$ and $B=\{b_1, \dots, b_n\}$ given in BHFP, where $a_i,b_i \in \{0, 1\}^d$ are binary vectors for all $i \in [n]$. Let $t, t'$ be the given threshold for BHFP and BHCP respectively. 

For any vector $b_j \in B$, let us define $\bar{b_j} \in \R^{d}$ as $(\Vec{1}-b_j)$, where $\Vec{1}:=(1,1,\dots,1) \in \R^d$. Now define $\bar{B}$ as the set of $\bar{b_j}$. Because the overall dimension of $B$ and $\bar{B}$ is $nd$, this process takes $O(nd)$ time. The squared $\ell_2$ distance between ${a_i}$ and $\bar{b_j}$ for any $i,j \in [n]$ is

\[ \|{a_i}-\bar{b_j}\|_2^2 = \|{a_i}-(\Vec{1}-b_j)\|_2^2 = d-\|{a_i}-{b_j}\|_2^2. \]

$\|a_i-\bar{b_j}\|_2^2 < d-t^2+1$ if and only if $\|{a_i}-{b_j}\|_2 \ge t$. Now if we run the algorithm for BHCP with the threshold $t' = \sqrt{d-t^2+1}$ on the sets ${A}$ and $\bar{B}$ to find a pair ${a} \in \bar{A}$ and $\bar{b} \in \bar{B}$ that satisfy $\|{a}-\bar{b}\|_2 < \sqrt{d-t^2+1} $, then we can conclude that there is a pair $a \in A$ and $b \in B$ that satisfies $\norm{a-b}_2 \geq t$.

Thus, if there is an algorithm for BHCP that runs in $O(n^{2-\epsilon})$ time, this entire process takes at most $O(nd + n^{2-\epsilon}) = O(n^{2-\epsilon})$ as long as $d = o(n^{1-\epsilon})$, and this algorithm decides if there exists at least a pair of vectors $a \in A$ and $b \in B$ such that $\norm{a-b}_2 \geq t$ in $O(n^{2-\epsilon})$ time, which contradicts the hardness result of BHFP. As a result, for at least one $t'$, BHCP problem cannot be solved in $O(n^{2-\epsilon})$ time.
\end{proof}

\section{Proof for Hardness of Sliding Window Dot-Product Self-Attention}
\label{app:sliding-window-hardness}

\paragraph{Theorem 3.}\textit{
Assume SETH. Let $f: \{\R^d, \R^d\} \longrightarrow \R$ as $f(x, y) = e^{ x^Ty}$. For set $Q$ of vectors $Q_1, \dots, Q_n$, we define the matrix $S \in \R^{n \times n}$ as
\begin{equation*}
 (S)_{ij}=\begin{cases} 
    f(Q_i,Q_j) & |i-j| \le w/2  \\
    0 & \text{otherwise} 
    \end{cases}
\end{equation*}
Also for any matrix $M \in \R^{n \times n}$, let $h(M)=M$. Let $Y=h(S) \cdot V \in \R^{n \times d_v}$ be a self-attention.
Then for any $\epsilon > 0$, computing a matrix $\hat{Y} \in \R^{n \times d_v}$ that satisfies any of the following conditions requires $\Omega((nw)^{1-\epsilon})$ time when $d_q = \omega(\log \sqrt{nw})$ and $w=\omega(d_q)$.
\begin{enumerate}
    \item $\hat{Y} = Y$(exact).
    \item $|\hat{Y}_{ij} - Y_{ij}| \le \mu |Y_{ij}|$ for all $i \in [n]$ and $j \in [d_v]$ where $0 \le \mu < 1 $(multiplicative approximation).
    \item $|\hat{Y}_{ij} - Y_{ij}| \le \mu$ for all $i \in [n]$ and $j \in [d_v]$ where $0 \le \mu$(additive approximation).
\end{enumerate}
}

\begin{proof}

Consider the two sets $A=\{a_1, \dots, a_{k}\}$ and $B=\{b_1, \dots, b_{k}\}$  given in TVPP, where $k=\sqrt{nw}$ and $a_i,b_i \in \{0, 1\}^d$ are binary vectors for all $i \in [k]$. The problem is to decide if there exists at least a pair $a \in A$ and $b \in B$ that satisfies $a^T b \geq t$ for a given $t \in [d]$.

We construct our matrices $Q$ and $V$ in the following way. Firstly, define the function $(\alpha \pmod{k}):= m$ if $\alpha \equiv m \pmod{k}$ and $m \in [k]$. Create the matrix $Q \in \R^{3n \times d}$ with its rows with even indices as $Q_{2 \alpha}=C b_{(\alpha \pmod{k})}$, and its rows with odd indices as $Q_{2 \alpha - 1}=a_{(\alpha \pmod{k}) + w \lfloor \frac{2\alpha-1}{2k}\rfloor}$. Thus, this process takes $O(n d)$ time. By this construction, each $(a_i,b_j)$ pair is found at a distance $w$.

Also, select V as concatenation of vector $(1,0,\dots,1,0) \in \R^{3n}$.

With this we have $Y =  h(S) \cdot V \in \R^{3n \times 1}$. Since $Y$ is a vector, we abuse the notation again and define $Y_{i}$ as the $i^{th}$ element of $Y$. 

Now consider the first $n$ even rows of $Y$. Because odd rows of matrix $Q$ is from set $A$, and even rows of matrix $Q$ is from set $B$, the even rows of $Y$ becomes the summation of the exponential of the $C$ times of the dot products of $w$ different $(a_i,b_j)$ pairs by the definition of vector $V$. 

Also, each $(a_i,b_j)$ pair is found at a distance $w$, so that the exponential of the $C$ times dot products of all pairs appear in the sliding window attention score matrix and contributes $Y_{2l}$ for at least an $l \in [n]$ value.

Now we check the magnitudes of $Y_{2l}$ for $l \in [n]$ in order to distinguish between true and false cases in TVPP. It takes $O(n)$ time to check these $n$ values. First, we focus on the exact computation. Consider the following two cases.

\textbf{Case 1.} There are no pairs $a \in A$ and $b \in B$ with $(a)^Tb \ge t$, that is for all $i, j \in [k]$, we have $(a_i)^Tb_j \le t-1$, and $e^{C (a_i)^Tb_j} \le e^{C(t-1)} $. Then for all $l \in [n]$, we have $Y_{2l} \le w e^{C(t-1)} := \delta$.

\textbf{Case 2.} There is a pair $a \in A$ and $b \in B$ with $a^Tb \ge t$, that is for some $i, j \in [k]$, we have $(a_i)^Tb_j \ge t$, and $e^{C (a_i)^Tb_j} \ge e^{Ct}$. Because any pair appears in some element of $Y$, we have $Y_{2l} \ge e^{Ct} :=\Delta$ for some odd $l \in [n]$.

In order to distinguish between two cases, it is sufficient to have $\Delta > \delta$. This holds with $C = 2 \log w$. 

Now let us consider multiplicative approximation error. With a $\mu$-multiplicative factor, if there are no pairs $a \in A$ and $b \in B$ with $a^Tb \ge t$, then we have for all $l \in [n], \hat{Y}_{2l} \le (1+\mu) Y_{2l} \le (1+\mu) w e^{C(t-1)} := \hat{\delta}$. On the other hand, if there is a pair $a \in A$ and $b \in B$ with $a^Tb \ge t$, then for some $l \in [n]$, we have $\hat{Y}_{2l} \ge (1-\mu) Y_{2l} \ge (1-\mu) e^{Ct} := \hat{\Delta}$. In order to distinguish between two cases, it is sufficient to have $\hat{\Delta} > \hat{\delta}$ and this inequality holds with $C = 2 \log (\frac{1+\mu}{1-\mu}w)$.

Finally we look at additive approximation error. With a $\mu$-additive factor, if there are no pairs $a \in A$ and $b \in B$ with $a^Tb \ge t$, then we have for all $l \in [n], \hat{Y}_{2l} \le Y_{2l} + \mu \le  w e^{C(t-1)} + \mu := \hat{\delta}$. On the other hand, if there is a pair $a \in A$ and $b \in B$ with $a^Tb \ge t$, then for some $l \in [n]$, we have $\hat{Y}_{2l} \ge  Y_{2l} -\mu \ge e^{Ct} - \mu := \hat{\Delta}$. In order to distinguish between two cases, it is sufficient to have $\hat{\Delta} > \hat{\delta}$ and this inequality holds with $C = 2 \log (w + 2 \mu)$.

Thus, if there is an algorithm for computing self-attention up to an element-wise multiplicative or additive error $\mu$ that runs in $O(k^{2-\epsilon})$ time, this entire process takes at most $O(nd + k^{2-\epsilon}) = O(\sqrt{nw}^{2-\epsilon}) =O({(nw)}^{1-\epsilon})$ as long as $d = o(w^{1-\epsilon})$. Therefore, this algorithm decides if there exists at least a pair of vectors $a \in A$ and $b \in B$ such that $a^Tb \ge t$ in $O((nw)^{1-\epsilon})$ time, which contradicts the hardness result of TVPP (Lemma \ref{lemma:ovp-to-tvpp}). This completes the proof.

\end{proof}

A similar proof also applies for dilated sliding window \cite{longformer}, where the self-attention score is calculated as Theorem \ref{thm:sliding-window-exp-dot-product}. Also, when the self-attention score is the softmax dot product (where softmax is only applied to the window size in each row), one can prove $O(nw^{1-\epsilon})$ complexity by following the proof of Theorem \ref{thm:hardness-softmax-self-attn}. 

% \fd{ I strongly believe that when $4 m^2 \approx nw $, select m as the size of TVPP problem. A good selection of $Q$ with some repetitions of $a_1, ... \ a_m , b_1, ... b_m $. each odd entry is $a$ and, each odd entry is $b$, then selecting $V$ as $(1,0,1,0,1,...,0,1)$, we can still have a good indicator for this TVPP problem. So it requires for any $\epsilon$, $O(m^{2-\epsilon})=O((nw)^{1-\epsilon})$ time. 
% }

% \fd{Spesifically $m^2=\frac{n-1}{2} . 2 \lfloor \frac{w+2}{4} \rfloor $}

% \fd{Example: w=4, n=19, $Q=\{b_4,a_1,b_3,a_4,b_1,a_1,b_2,a_2,b_3,a_3,b_4,a_4,b_2,a_3,b_1,a_2,b_4\}$ and $V= (1,0,1,0,...1,0,1)$. Then we receive all the multiplications in even rows.
% }

% \fd{This means that dimension $d_v =1$ works, so it can be anything, not dependent on $n$ or $w$}

% \fd{But I could not end that, should I just mention? The problem turns a different question, series question.}

\section{Proofs for Hardness of Self-Attention Matrix \texorpdfstring{$S$}{} Approximation}
\label{S_Approximations}

\begin{theorem}\label{thm:hardness-dot-product-self-attn-again}
Assume SETH. For any $i, j \in [n]$, let $S_{ij} = f(Q_i, Q_j) = e^{Q_i^TQ_j}$ and for any matrix $M \in \R^{n \times n}$, let $h(M) = M$. $\hat{S} \in \R^{n \times n}$ satisfies any of the following conditions:
\begin{enumerate}
    \item $|\hat{S}_{ij} - S_{ij}| \le \mu |S_{ij}|$ for all $i \in [n]$ and $j \in [d_v]$ where $0 \le \mu < 1 $(multiplicative approximation).
    \item $|\hat{S}_{ij} - S_{ij}| \le \mu$ for all $i,j \in [n]$, where $0 \le \mu$(additive approximation).
\end{enumerate}
Let $Y = h(\hat{S}).V \in \R^{n \times d_v}$ be a self-attention. Then for any $\epsilon > 0$, computing self-attention ${Y} \in \R^{n \times d_v}$ requires $\Omega(n^{2-\epsilon})$ time when $d_q = \omega(\log n)$.
\end{theorem}

\begin{proof}

Consider two sets $A=\{a_1, \dots, a_n\}$ and $B=\{b_1, \dots, b_n\}$  given in TVPP, where $a_i,b_i \in \{0, 1\}^d$ are binary vectors for all $i \in [n]$. The problem is to decide if there exists at least a pair  $a \in A$ and $b \in B$ that satisfies $a^Tb \geq t$ for a given $t \in [n]$.

Suppose that the matrices $Q$ and $V$ are constructed according to TVPP vector gadgets described in the section \ref{sec:vector-gadgets}. With this we have $Y =  h(\hat{S}).V \in \R^{n \times 1}$. Since $Y$ is a vector, we abuse the notation again and define $Y_i$ as the $i^{th}$ element of $Y$. Now consider the first $n$ elements of $Y$. Since $Y = SV$, we have that 
\[ Y_i = \sum_{j=1}^n \hat{S}_{ij} \text{ for any } i \in [n]. \]

Now we check the magnitude of $Y_i, i \in [n]$ in order to distinguish between true and false cases in TVPP. It takes $O(n)$ time to check each $Y_i, i \in [n]$. First we focus on the multiplicative error approximation. Consider the following two cases.

\textbf{Case 1.} There are no pairs $a \in A$ and $b \in B$ with $a^Tb \ge t$, that is for all $i, j \in [n]$, we have $a_i^Tb_j \le t-1$, and $S_{ij}=e^{C a_i^Tb_j} \le e^{C(t-1)}$, so $\hat{S}_{ij} \le (1+\mu) S_{ij} = (1+\mu) e^{C a_i^Tb_j} \le (1+\mu) e^{C(t-1)} $. Then for all $l \in [n]$, $Y_l \le n (1+\mu) e^{C(t-1)} := \delta$.

\textbf{Case 2.} There is a pair $a \in A$ and $b \in B$ with $a^Tb \ge t$, that is for some $i, j \in [n]$, we have $a_{i}^Tb_{j} \ge t$, and $S_{ij}=e^{C a_{i}^Tb_{j}} \ge e^{Ct}$, so $\hat{S}_{ij} \ge (1-\mu) S_{ij} = (1-\mu) e^{C a_i^Tb_j} \ge (1-\mu) e^{Ct} $. Then for some $l \in [n]$, we have $Y_l \ge (1-\mu) e^{Ct} + n-1 :=\Delta$ 

In order to distinguish between two cases, it is sufficient to have $\Delta > \delta$. This holds with $C = 2 \log (\frac{1+\mu}{1-\mu}n)$.

Now let us look at the additive error approximation. With a $\mu$ additive factor,  if there are no pairs $a \in A$ and $b \in B$ with $a^Tb \ge t$, then we have for all $l \in [n], {Y}_l \le n e^{C(t-1)}  + n\mu := \hat{\delta}$. On the other hand, if there is a pair $a \in A$ and $b \in B$ with $a^Tb \ge t$, then for some $l \in [n]$, we have ${Y}_l \ge e^{Ct} + n - 1 - \mu := \hat{\Delta}$. In order to distinguish between two cases, it is sufficient to have $\hat{\Delta} > \hat{\delta}$ and this inequality holds with $C = 2 \log (n + 2 \mu)$.

Thus, if there is an algorithm for computing self-attention up to an element-wise multiplicative or additive error $\mu$ that runs in $O(n^{2-\epsilon})$ time, this entire process takes at most $O(n^{2-\epsilon})$, and this algorithm decides if there exists at least a pair of vectors $a \in A$ and $b \in B$ such that $a^Tb \ge t$ in $O(n^{2-\epsilon})$ time, which contradicts the hardness result of TVPP (Lemma \ref{lemma:ovp-to-tvpp}). This completes the proof.
\end{proof}

Similar proofs also work to show the quadratic complexity of multiplicative approximation to the $S$ of softmax dot-product self-attention, and $\ell_2$-self-attention directly from Theorem \ref{thm:hardness-softmax-self-attn} and Theorem \ref{thm:hardness-l2-self-attn}. And also by Theorem \ref{thm:sliding-window-exp-dot-product}, one can show the quadratic complexity of additive and multiplicative approximation to the $S$ of sliding window dot-product self-attention.

\section{Proofs of Lemmas for Polynomial Approximations of Self-Attention
}

\paragraph{Lemma 4.} \textit{
Let $p$ be an integer $\ge 0$ and let $S_{ij}$ be $C \cdot (Q_i^TK_j)^p$ where $i, j \in [n]$ and $C$ is a constant. Then $SV$ can be computed in $O(nd_q^pd_v)$ time.}

\begin{proof}
We omit the constant $C$ in this proof since we can multiply any of the matrices $Q, K, V$ by $C$ in $O(nd_q)$ or $O(nd_v)$ time before the rest of the computation. 

When $p=1$, $SV = QK^TV$ thus $SV$ can be trivially computed by first computing $K^TV$ by multiplying $K^T$ and $V$ with $O(d_qnd_v)$ time, then by multiplying $Q$ and $K^TV$ with $O(nd_qd_v)$ time.

For $p \ge 2$, Consider a fixed $(i, j)$ pair. Now
\begin{equation}\label{eq:expansion-pth-power}
\begin{split}
    S_{ij} & = (Q_i^TK_j)^p =  \Big( \sum_{r=1}^{d_q} x_r y_r \Big)^p = \Big( x_1y_1 + \dots + x_{d_q}y_{d_q} \Big)^p =  \sum_{r_1, \dots, r_p = 1}^{d_q} x_{r_1}y_{r_1} \dots x_{r_p}y_{r_p} \\
           & = \sum_{r_1, \dots, r_p = 1}^{d_q} (x_{r_1} \dots x_{r_p})(y_{r_1} \dots y_{r_p})
\end{split}
\end{equation}
where $x = Q_i, y = K_j$ and $x_r, y_r$ denote the $r^{th}$ element of the corresponding vectors. Given a vector $v \in \R^{d_q}$, let us define a function $\alpha: \R^{d_q} \rightarrow \R^{d_q^p}$ where elements of $\alpha(v)$ are computed by element-wise multiplication $v_{r_1} \dots v_{r_p}, r_1, \dots, r_p \in [d_q]$(ordered $p$-permutations of $d_q$ with replacement). With this, let us define $\hat{Q} \in \R^{n \times d_q^p}$(and $\hat{K}$) where each row $i$ is computed with $\alpha(Q_i)$(respectively $\alpha(K_j)$ for $\hat{K}$).  $\hat{Q}, \hat{K}$ matrices can be computed in $O(nd_q^p)$ time. From Eq.\ref{eq:expansion-pth-power}, it is evident that $(Q_i^TK_j)^p = \hat{Q}_i^T \hat{K}_j$. Now, similar to the case $p=1$ one can compute $SV$ by first computing $\hat{K}^TV$ by multiplying $\hat{K}^T$ and $V$ with $O(d_q^pnd_v)$ time, then by multiplying $\hat{Q}$ and $\hat{K}^TV$ with $O(nd_q^pd_v)$ time. This completes the proof.

\end{proof}

\paragraph{Lemma 5.}\textit{
Let $p$ be an integer $\ge 0$. Then for all $i \in [n]$, $\sum_{\hat{j}=1}^n C \cdot (Q_i^TK_{\hat{j}})^p, \hat{j} \in [n]$ where $C$ is a constant can be computed in $O(nd_q^p)$ time.}

\begin{proof}

We omit the constant $C$ in this proof since we can multiply any of the matrices $Q, K$ by $C$ in $O(nd_q)$ time before the rest of the computation.

Let $\hat{Q}, \hat{K} \in \R^{n \times d_q^p}$ be the matrices computed by applying $\alpha(.)$ on rows of $Q, V$ as stated in the proof of Lemma \ref{lem:pth-order-poly}(from Eq. \ref{eq:expansion-pth-power}). The time complexity of computing $\hat{Q}, \hat{K}$ is $O(nd_q^p)$.

Now we have $\sum_{\hat{j}=1}^n (Q_i^TK_{\hat{j}})^p = \sum_{\hat{j}=1}^n \hat{Q_i}^T \hat{K_{\hat{j}}}  = \hat{Q_i}^T \sum_{\hat{j}=1}^n \hat{K_{\hat{j}}}$. Let $A = \sum_{\hat{j}=1}^n \hat{K_{\hat{j}}}$ and $A$ can be computed in $O(nd_q^p)$ time(taking the summation of $n$ $d_q^p$ size vectors). We store the value of $A$ in the memory and reuse it. One can simply compute the inner product of $\hat{Q_i}^T$ and $A$ in $O(d_q^p)$ time per $i$. For all $i \in [n]$, this takes $O(nd_q^p)$ time which gives the desired time complexity.

\end{proof}

\section{ Proof for Hardness of \texorpdfstring{$\ell_2 \textnormal{-} $}{}Self-Attention}
\label{app:proof_l2_self_attention}

\paragraph{Theorem 4.}\textit{
Assume SETH. For any $i, j \in [n]$, let $S_{ij} = f(Q_i, Q_j) = e^{C. \norm{Q_i-Q_j}_2^2}$ where $C$ is a parameter and for any matrix $M \in \R^{n \times n}$, let $h(M) = M$. Let $Y = h(S).V \in \R^{n \times d_v}$ be a self-attention. Then for any $\epsilon > 0$, computing a matrix $\hat{Y} \in \R^{n \times d_v}$ that satisfies any of the following conditions requires $\Omega(n^{2-\epsilon})$ time when $d_q = \omega(\log n)$.
\begin{enumerate} [nosep,leftmargin=*]
    \item $\hat{Y} = Y$(exact).
    \item $|\hat{Y}_{ij} - Y_{ij}| \le \mu |Y_{ij}|$ for all $i \in [n]$ and $j \in [d_v]$ where $0 \le \mu < 1 $(multiplicative approximation).
\end{enumerate}
}

\begin{proof}

Consider two sets $A=\{a_1, \dots, a_n\}$ and $B=\{b_1, \dots, b_n\}$  given in BHCP, where $a_i,b_i \in \{0, 1\}^d$ are binary vectors for all $i \in [n]$. The problem is to decide if there exists at least a pair  $a \in A$ and $b \in B$ that satisfies $\norm{a-b}_2^2 < t$ for a given $t \in [n]$.

Suppose that the matrices $Q$ and $V$ are constructed according to BHCP vector gadgets described in the section \ref{sec:vector-gadgets}. With this we have $Y =  h(S).V \in \R^{n \times 1}$. Since $Y$ is a vector, we abuse the notation again and define $Y_i$ as the $i^{th}$ element of $Y$. Now consider the first $n$ elements of $Y$. Since $Y = SV$, we have that 

\[ Y_i = \sum_{j=1}^n e^{- C\norm{a_i-b_j}_2^2} \text{ for any } i \in [n]. \]

Now we check the magnitude of $Y_i, \in [n]$ in order to distinguish between true and false cases in BHCP. It takes $O(n)$ time to check each $Y_i, i \in [n]$. First we focus on the exact computation. Consider the following two cases.

\textbf{Case 1.} There are no pairs $a \in A$ and $b \in B$ with $\norm{a_i-b_j}_2^2 \le t$, that is for all $i, j \in [n]$, we have $\norm{a_i-b_j}_2^2 \ge t$, and $e^{- C\norm{a_i-b_j}_2^2} \le e^{-Ct} $. Then for all $l \in [n]$, $Y_l \le n e^{-Ct} := \delta$.

\textbf{Case 2.} There is a pair $a \in A$ and $b \in B$ with $\norm{a-b}_2^2 < t$, that is for some $i, j \in [n]$, we have $\norm{a_i-b_j}_2^2 \le t-1$, and $e^{- C\norm{a_i-b_j}_2^2} \ge e^{-C(t-1)}$. Thus for some $l \in [n]$, we have $Y_l \ge e^{-C(t-1)} :=\Delta$ 

In order to distinguish between two cases, it is sufficient to have $\Delta > \delta$. This holds with $C = 2 \log n$. 

Now let us look at the multiplicative error approximation. With a $\mu$ multiplicative factor, if there are no pairs $a \in A$ and $b \in B$ with $a^Tb \ge t$, then we have for all $l \in [n], \hat{Y}_l \le (1+\mu) Y_l \le (1+\mu) n e^{-Ct} := \hat{\delta}$. On the other hand, if there is a pair $a \in A$ and $b \in B$ with $a^Tb \ge t$, then for some $l \in [n]$, we have $\hat{Y}_l \ge (1-\mu) Y_l \ge (1-\mu) e^{-C(t-1)} := \hat{\Delta}$. In order to distinguish between two cases, it is sufficient to have $\hat{\Delta} > \hat{\delta}$ and this inequality holds with $C = 2 \log (\frac{1+\mu}{1-\mu}n)$.

Thus, if there is an algorithm for computing self-attention up to an element-wise multiplicative error $\mu$ that runs in $O(n^{2-\epsilon})$ time, this entire process takes at most $O(n + nd + n^{2-\epsilon}) = O(n^{2-\epsilon})$ as long as $d = o(n^{1-\epsilon})$, and this algorithm decides if there exists at least a pair of vectors $a \in A$ and $b \in B$ such that $\norm{a_i-b_j}_2^2 \ge t$ in $O(n^{2-\epsilon})$ time, which contradicts the hardness result of BHCP (Lemma \ref{lemma:bhfp-to-bhcp}). This completes the proof.
\end{proof}

Here in the proof we set $C = 2 \log(\frac{2(1+\mu)}{1-\mu} n)$ however, the parameter $C$ in the RBF kernel is predefined, therefore the hardness result is valid only conditioned on this specific $C$. We can bypass this by simply defining $C = \log(\frac{2(1+\mu)}{1-\mu} n) = \alpha \beta$, then setting $\alpha$ to the predefined constant of the RBF function and solving this equation for $\beta$. After this, all that remains is modifying the BHCP vector gadget(see section \ref{sec:vector-gadgets}) by multiplying each vector $Q_i, i \in [2n]$ by the scaler $\sqrt{\beta}$ in $O(n)$ time, and we obtain the desired hardness result for general RBF kernel.

\section{Multi-Head Self-Attention}
\label{app:multihead}

\paragraph{Lemma 6.} $k$ parallel OVP (or TVPP, or BHCP, or BHFP) problems (each has two sets with n binary vectors) require $O(kn^{2-\epsilon})$ time for any $\epsilon$.

\begin{proof}
Suppose there is an algorithm for k parallel OVP (TVPP, BHCP, BHFP) problems better than $O(kn^{2-\epsilon})$ time. 

Say $A_i$ and $B_i$ are the sets of binary vectors with size $n$ for any $i \in [k]$.  

For each $t \in [k]$, look at these k parallel OVP (TVPP, BHCP, BHFP) problems:
$$(A_1, B_t) ,(A_2, B_{t+1}), \cdots ,(A_k, B_{t+k-1}) \textnormal{, where } B_{l+k}=B_l$$ 
Because of the assumption, there is an algorithm better than $O(kn^{2-\epsilon})$ time. 

So that, there is an $O(k \times kn^{2-\epsilon})$-time algorithm that solves OVP problem of ($A=A_1 \cup \cdots A_k$, $B=B_1 \cup \cdots B_k$). In other words for the OVP (TVPP, BHCP, BHFP) problem with sets of size $nk$ binary vectors has an algorithm in $O(kn^{2-\delta})$-time (by selecting $\delta=\epsilon/2$ and $k<n$). This contradicts SETH. As a result,  $k$ parallel OVP (TVPP, BHCP, BHFP) problems require $O(kn^{2-\epsilon})$ time for any $\epsilon$.

This lemma proves that the ``direct sum" of computational complexity for OVP (TVPP, BHCP, BHFP) problems is valid.

\end{proof} 

\section{Discussion of Additive Approximation}
\label{app:Additive}

This part is depends on for a given $C$ value, and the selected $\delta$ and $\Delta$ values, the difference $\Delta-\delta$ is positive. So that, this allows us to select an additive error $\mu$.

The following theorem shows that the we cannot reach better elementwise additive approximation than $e^{-2d \log(n+2)}$ for the $\ell_2$ distance self-attention.

\paragraph{Theorem 5.} Assume SETH. For any $i,j \in [n]$, let $S_{ij}=f(Q_i,Q_j)=e^{-C \| Q_i-Q_j  \|_2^2  }$, where $C$ is a parameter and for any matrix $M\in R^{n\times n}$. Let $h(M)=M$ and $Y=h(S)$. $V \in \R^{n\times d_v }$ be self-attention. Then for any $\epsilon>0$, computing a matrix $\hat{Y} \in R^{n\times d_v }$  that satisfies the following
condition requires $\Omega(n^{2-\epsilon})$ time when $d_q=\omega(\log n)$:
$|\hat Y_{ij}-Y_{ij} |\leq \mu$ for all $i\in [n]$ and $j\in [d_v]$, where $0 \leq\mu\leq e^{-2d \log(n+2) }=(n+2)^{-2d}$ (additive approximation)

\begin{proof}Consider two sets $A = \{a_1,\dots,a_n\}$ and $B= \{b_1,\dots,b_n\}$ given in BHCP, where $a_i,b_i  \in \{0,1\}^d$ are binary vectors for all $i\in [n]$. The problem is to decide if there exists at least a pair $a\in A$ and $b\in B$ that satisfies $\|a-b\|_2^2 <t$ for a given $t\in [n]$.

Suppose that the matrices $Q$ and $V$ are constructed according to BHCP vector gadgets described in the section 4.2. With this we have  $Y=h(S)\cdot V \in R^{n\times 1}$. Since $Y$ is a vector, we abuse the notation again and define $Y_i$ as the $i$th element of $Y$. Now consider the first n elements of $Y$. Since $Y=SV$, we have that $Y_i=\sum_{j=1}^n e^{-C \| a_i-b_j  \|_2^2 }$ for any $i,j \in [n]$.

Case 1. There are no pairs $a\in A$ and $b\in B$ that satisfies $\|a-b \|_2^2 <t$,  that is for all $i,j\in[n]$, we have $\|a-b \|_2^2 \geq t$, and $e^{-C \| a_i-b_j  \|_2^2 }\leq e^{-Ct}$. Then for all  $l\in [n]$,  $Y_l\leq n e^{-Ct}$. 
It is given that $|\hat{Y}_l-Y_l\leq \mu$, so $\hat{Y}_l\leq Y_l+\mu\leq n e^{-Ct}+\mu \leq n e^{-Ct} + e^{-2d \log(n+2)}=:\delta$.

Case 2. There is a pair  $a\in A$ and $b\in B$ with  $\|a-b \|_2^2 <t$, that is for some $i,j\in [n]$, we
have $\|a-b \|_2^2 \leq t-1$, and $e^{-C \| a_i-b_j  \|_2^2 }\geq e^{-C(t-1)}$. Thus for some $l\in[n]$, $Y_l\geq ne^{-C(t-1)}$.  
It is given that $|\hat{Y}_l-Y_l\leq \mu$ , so $\hat{Y}_l\geq Y_l-\mu\leq n e^{-Ct}-\mu\geq n e^{-Ct}-e^{-2d \log(n+2)}=:\Delta $.

In order to distinguish between two cases, it is sufficient to have $\Delta>\delta$. This holds with $C=2\log(n+2)$.

Thus, if there is an algorithm for computing self-attention up to an element-wise additive error $\mu \leq e^{-2d \log(n+2) }$ that runs in  $O(n^{2-\epsilon})$ time, this entire process takes at most $O(n+nd+n^{2-\epsilon})=O(n^{2-\epsilon})$  as long $d=O(n^{1-\epsilon})$, and this algorithm decides if there exists at least a pair of vectors $a\in A$ and $b \in B$ with  $\|a-b\|_2^2<t$ in $O(n^{2-\epsilon})$  time, which contradicts the hardness result of BHCP (Lemma 3). This completes the proof.
\end{proof} 

This additive approximation error can be improved slightly, but its order must be $O(n^{-d})$.

The following theorem shows that the we cannot reach better elementwise additive approximation than $e^{-3d \log(n)-3d^2}=n^{-3d}\cdot e^{-3d^2}$ for the softmax dot-product self-attention.

\paragraph{Theorem 6.} Assume SETH. For any $i,j \in [n]$, let $S_{ij}=f(Q_i,Q_j)=e^{ Q_i^T Q_j }$, and for any matrix $M\in \R^{n\times n}$ , let $h(M)\in \R^{n\times n}$ be the matrix where for all $i,j \in [n]$, $\{h(M)\}_{ij}=  \frac{M_{ij}}{\sum_{k_1}^n M_{ik}}$. Let $Y=h(S)\cdot V \in \R^{n\times d_v }$ be self-attention. Then provided $d_q=\omega(\log n)$, for any $\epsilon>0$, computing a matrix $\hat Y= \R^{n\times  d_v} $  that satisfies the following condition requires $\Omega (n^{2-\epsilon})$ time:

$|\hat Y_{ij}-Y_{ij}|\leq \mu$ for all $i\in[n]$ and $j\in[d_v]$, where $0 \leq \mu \leq e^{-3d \log(n)-3d^2  }$ (additive approximation)

\begin{proof}The proof is technically similar. Suppose that the matrices $Q$ and $V$ are constructed according to TVPP vector gadgets described in Section 4.2. With this we have $Y=h(S)\cdot V\in \R^{n\times 1}$. Consider the first $n$ elements of $Y$. Since $h(\cdot)$ act as the row-wise softmax function, we have 

\begin{align*}
Y_i = \sum_{j=n+1}^{2n} S_{ij} &=\sum_{j=1}^n \frac{e^{Ca_i^Tb_j}}{\sum_{k=1}^ne^{Ca_i^Ta_k}+\sum_{k=1}^ne^{Ca_i^Tb_k}}\\
&=\frac{\sum_{j=1}^n e^{Ca_i^Tb_j}}{\sum_{k=1}^ne^{Ca_i^Ta_k}+\sum_{k=1}^ne^{Ca_i^Tb_k}}
\end{align*}

Case 1. There are no pairs $a\in A$ and $b\in B$ that satisfies $a^T b\geq t$, that is for all $i,j\in[n]$, we have $a^T b\leq t-1$, and $\sum_{j=1}^n e^{Ca_i^Tb_j} \leq ne^{C(t-1)}$. For a function $\frac{x}{x+y}$, the maximum value is achieved at maximum $x$ and minimum $y$ values. Thus, for all $l\in [n]$,  $Y_l\leq \frac{ne^{C(t-1)}}{n e^{C(t-1)}+n}=\frac{e^{C(t-1)}}{e^{C(t-1)}+1}$.
It is given that $|\hat Y_l-Y_l|\leq \mu$ , so $\hat Y_l \leq Y_l+\mu \leq\frac{e^{C(t-1)}}{e^{C(t-1)}+1}+\mu\leq \frac{e^{C(t-1)}}{e^{C(t-1)}+1}+e^{-3d \log(n)-3d^2 }=:\delta$.

Case 2. There is a pair $a\in A$ and $b\in B$ with $a^T b\geq t$, that is for some $i,j\in [n]$, we have $a^T b\geq t$. Then the row sum corresponding to that $i,j$ pair is $\sum_{j=1}^n e^{Ca_i^T b_j }\geq e^{Ct}+(n-1)e^0=e^{Ct}+n-1$. For a function $\frac{x}{x+y}$, the maximum value is achieved at minimum $x$ and maximum $y$ values. Thus, for some $l\in[n]$, we have $Y_l\geq \frac{ e^{Ct}+(n-1)}{e^{Ct}+(n-1)+n e^d }$.

\end{proof}

\section{Discussion of Remark in Section 4.4}
\label{app:Remark}

In our results, we show the quadratic hardness of multiplicative error approximations self-attention matrix elements for dot-product softmax self-attention mechanism. One assumption we make on the approximation factor $\mu$ is that $\mu < 1$. Consider the value of the parameter $C = \log(\frac{2(1+\mu)}{1-\mu} n) + d$ in the proof of theorem \ref{thm:hardness-softmax-self-attn}. Notice that when $\mu = 1$, $C$ is not defined. In fact, having $\mu = 1$ implies that the $|\hat{Y}_l - Y_l| = |Y_l|$ for all $l \in [n]$, therefore one can approximate every entry of $\hat{Y}$ by $0$ in $O(n)$ time while satisfying this condition. However, as mentioned in the remark, one can set $\mu$ close to $1$ by setting it as $1-\frac{1}{n^x}$ for a constant $x$, while maintaining the condition on $C$ in the same order.

\end{document}

%% file: header.tex
\usepackage{amsmath, amsthm, amssymb}
\usepackage{nicefrac}
\usepackage{xcolor}

\usepackage{microtype}
\usepackage{graphicx}
\usepackage{multirow}
\usepackage{booktabs} % for professional tables
\usepackage{algorithm}
\usepackage{algorithmic}

\usepackage{caption}
\usepackage{subcaption}
\usepackage{float}
\usepackage{breakcites}

\usepackage{hyperref}

\newtheorem{theorem}{Theorem}
\newtheorem{lemma}{Lemma}
\newtheorem{corollary}{Corollary}

\newtheorem{defn}{Definition}

\newcommand{\norm}[1]{\left \lVert #1 \right \rVert}

\newcommand{\R}{\mathbb{R}}

\newcommand{\softmax}{\textnormal{softmax}}
\newcommand{\Attention}{\textnormal{Attention}}
\usepackage{pifont}

\newcommand{\xmark}{\ding{55}}

%% file: main.bbl
\begin{thebibliography}{10}
\providecommand{\url}[1]{#1}
\csname url@samestyle\endcsname
\providecommand{\newblock}{\relax}
\providecommand{\bibinfo}[2]{#2}
\providecommand{\BIBentrySTDinterwordspacing}{\spaceskip=0pt\relax}
\providecommand{\BIBentryALTinterwordstretchfactor}{4}
\providecommand{\BIBentryALTinterwordspacing}{\spaceskip=\fontdimen2\font plus
\BIBentryALTinterwordstretchfactor\fontdimen3\font minus
  \fontdimen4\font\relax}
\providecommand{\BIBforeignlanguage}[2]{{%
\expandafter\ifx\csname l@#1\endcsname\relax
\typeout{** WARNING: IEEEtran.bst: No hyphenation pattern has been}%
\typeout{** loaded for the language `#1'. Using the pattern for}%
\typeout{** the default language instead.}%
\else
\language=\csname l@#1\endcsname
\fi
#2}}
\providecommand{\BIBdecl}{\relax}
\BIBdecl

\bibitem{vaswani2017attention}
A.~Vaswani, N.~Shazeer, N.~Parmar, J.~Uszkoreit, L.~Jones, A.~N. Gomez,
  {\L}.~Kaiser, and I.~Polosukhin, ``Attention is all you need,''
  \emph{Advances in neural information processing systems}, vol.~30, 2017.

\bibitem{kenton2019bert}
J.~D. M.-W.~C. Kenton and L.~K. Toutanova, ``Bert: Pre-training of deep
  bidirectional transformers for language understanding,'' in \emph{Proceedings
  of NAACL-HLT}, 2019, pp. 4171--4186.

\bibitem{dosovitskiy2020image}
A.~Dosovitskiy, L.~Beyer, A.~Kolesnikov, D.~Weissenborn, X.~Zhai,
  T.~Unterthiner, M.~Dehghani, M.~Minderer, G.~Heigold, S.~Gelly \emph{et~al.},
  ``An image is worth 16x16 words: Transformers for image recognition at
  scale,'' in \emph{International Conference on Learning Representations},
  2021.

\bibitem{khan2021transformers}
S.~Khan, M.~Naseer, M.~Hayat, S.~W. Zamir, F.~S. Khan, and M.~Shah,
  ``Transformers in vision: A survey,'' \emph{ACM Computing Surveys (CSUR)},
  2021.

\bibitem{jumper2021highly}
J.~Jumper, R.~Evans, A.~Pritzel, T.~Green, M.~Figurnov, O.~Ronneberger,
  K.~Tunyasuvunakool, R.~Bates, A.~{\v{Z}}{\'\i}dek, A.~Potapenko
  \emph{et~al.}, ``Highly accurate protein structure prediction with
  alphafold,'' \emph{Nature}, vol. 596, no. 7873, pp. 583--589, 2021.

\bibitem{chen2021evaluating}
M.~Chen, J.~Tworek, H.~Jun, Q.~Yuan, H.~P. de~Oliveira~Pinto, J.~Kaplan,
  H.~Edwards, Y.~Burda, N.~Joseph, G.~Brockman \emph{et~al.}, ``Evaluating
  large language models trained on code,'' \emph{CoRR}, 2021.

\bibitem{radford2021learning}
A.~Radford, J.~W. Kim, C.~Hallacy, A.~Ramesh, G.~Goh, S.~Agarwal, G.~Sastry,
  A.~Askell, P.~Mishkin, J.~Clark \emph{et~al.}, ``Learning transferable visual
  models from natural language supervision,'' in \emph{International Conference
  on Machine Learning}.\hskip 1em plus 0.5em minus 0.4em\relax PMLR, 2021, pp.
  8748--8763.

\bibitem{alayrac2022flamingo}
J.-B. Alayrac, J.~Donahue, P.~Luc, A.~Miech, I.~Barr, Y.~Hasson, K.~Lenc,
  A.~Mensch, K.~Millican, M.~Reynolds \emph{et~al.}, ``Flamingo: a visual
  language model for few-shot learning,'' \emph{arXiv preprint
  arXiv:2204.14198}, 2022.

\bibitem{kitaev2019reformer}
N.~Kitaev, L.~Kaiser, and A.~Levskaya, ``Reformer: The efficient transformer,''
  in \emph{International Conference on Learning Representations}, 2020.

\bibitem{zaheer2020big}
M.~Zaheer, G.~Guruganesh, K.~A. Dubey, J.~Ainslie, C.~Alberti, S.~Ontanon,
  P.~Pham, A.~Ravula, Q.~Wang, L.~Yang \emph{et~al.}, ``Big bird: Transformers
  for longer sequences,'' \emph{Advances in Neural Information Processing
  Systems}, vol.~33, pp. 17\,283--17\,297, 2020.

\bibitem{choromanski2020rethinking}
K.~M. Choromanski, V.~Likhosherstov, D.~Dohan, X.~Song, A.~Gane, T.~Sarlos,
  P.~Hawkins, J.~Q. Davis, A.~Mohiuddin, L.~Kaiser \emph{et~al.}, ``Rethinking
  attention with performers,'' in \emph{International Conference on Learning
  Representations}, 2021.

\bibitem{katharopoulos2020transformers}
A.~Katharopoulos, A.~Vyas, N.~Pappas, and F.~Fleuret, ``Transformers are rnns:
  Fast autoregressive transformers with linear attention,'' in
  \emph{International Conference on Machine Learning}.\hskip 1em plus 0.5em
  minus 0.4em\relax PMLR, 2020, pp. 5156--5165.

\bibitem{soft}
J.~Lu, J.~Yao, J.~Zhang, X.~Zhu, H.~Xu, W.~Gao, C.~XU, T.~Xiang, and L.~Zhang,
  ``Soft: Softmax-free transformer with linear complexity,'' in \emph{Advances
  in Neural Information Processing Systems}, M.~Ranzato, A.~Beygelzimer,
  Y.~Dauphin, P.~Liang, and J.~W. Vaughan, Eds., vol.~34.\hskip 1em plus 0.5em
  minus 0.4em\relax Curran Associates, Inc., 2021, pp. 21\,297--21\,309.

\bibitem{chen2021skyformer}
Y.~Chen, Q.~Zeng, H.~Ji, and Y.~Yang, ``Skyformer: Remodel self-attention with
  gaussian kernel and nystr$\backslash$" om method,'' \emph{Advances in Neural
  Information Processing Systems}, vol.~34, 2021.

\bibitem{backurs2015edit}
A.~Backurs and P.~Indyk, ``Edit distance cannot be computed in strongly
  subquadratic time (unless seth is false),'' in \emph{Proceedings of the
  forty-seventh annual ACM symposium on Theory of computing}, 2015, pp. 51--58.

\bibitem{bringmann2014walking}
K.~Bringmann, ``Why walking the dog takes time: Frechet distance has no
  strongly subquadratic algorithms unless seth fails,'' in \emph{2014 IEEE 55th
  Annual Symposium on Foundations of Computer Science}.\hskip 1em plus 0.5em
  minus 0.4em\relax IEEE, 2014, pp. 661--670.

\bibitem{bringmann2015quadratic}
K.~Bringmann and M.~K{\"u}nnemann, ``Quadratic conditional lower bounds for
  string problems and dynamic time warping,'' in \emph{2015 IEEE 56th Annual
  Symposium on Foundations of Computer Science}.\hskip 1em plus 0.5em minus
  0.4em\relax IEEE, 2015, pp. 79--97.

\bibitem{bringmann2021fine}
K.~Bringmann, ``Fine-grained complexity theory: Conditional lower bounds for
  computational geometry,'' in \emph{Conference on Computability in
  Europe}.\hskip 1em plus 0.5em minus 0.4em\relax Springer, 2021, pp. 60--70.

\bibitem{backurs2017fine}
A.~Backurs, P.~Indyk, and L.~Schmidt, ``On the fine-grained complexity of
  empirical risk minimization: Kernel methods and neural networks,''
  \emph{Advances in Neural Information Processing Systems}, vol.~30, 2017.

\bibitem{rubinstein2018hardness}
A.~Rubinstein, ``Hardness of approximate nearest neighbor search,'' in
  \emph{Proceedings of the 50th annual ACM SIGACT symposium on theory of
  computing}, 2018, pp. 1260--1268.

\bibitem{brown2020language}
T.~Brown, B.~Mann, N.~Ryder, M.~Subbiah, J.~D. Kaplan, P.~Dhariwal,
  A.~Neelakantan, P.~Shyam, G.~Sastry, A.~Askell \emph{et~al.}, ``Language
  models are few-shot learners,'' \emph{Advances in neural information
  processing systems}, vol.~33, pp. 1877--1901, 2020.

\bibitem{liu2021swin}
Z.~Liu, Y.~Lin, Y.~Cao, H.~Hu, Y.~Wei, Z.~Zhang, S.~Lin, and B.~Guo, ``Swin
  transformer: Hierarchical vision transformer using shifted windows,'' in
  \emph{Proceedings of the IEEE/CVF International Conference on Computer
  Vision}, 2021, pp. 10\,012--10\,022.

\bibitem{yun2019transformers}
C.~Yun, S.~Bhojanapalli, A.~S. Rawat, S.~Reddi, and S.~Kumar, ``Are
  transformers universal approximators of sequence-to-sequence functions?'' in
  \emph{International Conference on Learning Representations}, 2019.

\bibitem{dong2021attention}
Y.~Dong, J.-B. Cordonnier, and A.~Loukas, ``Attention is not all you need: Pure
  attention loses rank doubly exponentially with depth,'' in
  \emph{International Conference on Machine Learning}.\hskip 1em plus 0.5em
  minus 0.4em\relax PMLR, 2021, pp. 2793--2803.

\bibitem{kim2021lipschitz}
H.~Kim, G.~Papamakarios, and A.~Mnih, ``The lipschitz constant of
  self-attention,'' in \emph{International Conference on Machine
  Learning}.\hskip 1em plus 0.5em minus 0.4em\relax PMLR, 2021, pp. 5562--5571.

\bibitem{wang2020linformer}
S.~Wang, B.~Z. Li, M.~Khabsa, H.~Fang, and H.~Ma, ``Linformer: Self-attention
  with linear complexity,'' \emph{arXiv preprint arXiv:2006.04768}, 2020.

\bibitem{longformer}
I.~Beltagy, M.~E. Peters, and A.~Cohan, ``Longformer: The long-document
  transformer,'' \emph{CoRR}, vol. abs/2004.05150, 2020.

\bibitem{roy2021efficient}
A.~Roy, M.~Saffar, A.~Vaswani, and D.~Grangier, ``Efficient content-based
  sparse attention with routing transformers,'' \emph{Transactions of the
  Association for Computational Linguistics}, vol.~9, pp. 53--68, 2021.

\bibitem{xiong2021nystromformer}
Y.~Xiong, Z.~Zeng, R.~Chakraborty, M.~Tan, G.~Fung, Y.~Li, and V.~Singh,
  ``Nystr{\"o}mformer: A nystr{\"o}m-based algorithm for approximating
  self-attention,'' in \emph{Proceedings of the AAAI Conference on Artificial
  Intelligence}, vol.~35, no.~16, 2021, pp. 14\,138--14\,148.

\bibitem{likhosherstov2021sub}
V.~Likhosherstov, K.~M. Choromanski, J.~Q. Davis, X.~Song, and A.~Weller,
  ``Sub-linear memory: How to make performers slim,'' \emph{Advances in Neural
  Information Processing Systems}, vol.~34, 2021.

\bibitem{peng2020random}
H.~Peng, N.~Pappas, D.~Yogatama, R.~Schwartz, N.~Smith, and L.~Kong, ``Random
  feature attention,'' in \emph{International Conference on Learning
  Representations}, 2021.

\bibitem{bhojanapalli2021leveraging}
S.~Bhojanapalli, A.~Chakrabarti, A.~Veit, M.~Lukasik, H.~Jain, F.~Liu, Y.-W.
  Chang, and S.~Kumar, ``Leveraging redundancy in attention with reuse
  transformers,'' \emph{arXiv preprint arXiv:2110.06821}, 2021.

\bibitem{seth}
R.~Impagliazzo and R.~Paturi, ``On the complexity of k-sat,'' \emph{Journal of
  Computer and System Sciences}, vol.~62, no.~2, pp. 367--375, 2001.

\bibitem{seth2}
R.~Impagliazzo, R.~Paturi, and F.~Zane, ``Which problems have strongly
  exponential complexity?'' \emph{Journal of Computer and System Sciences},
  vol.~63, no.~4, p. 512–530, 2001.

\bibitem{indyk2017beyond}
P.~Indyk, ``Beyond p vs. np: quadratic-time hardness for big data problems,''
  in \emph{Proceedings of the 29th ACM Symposium on Parallelism in Algorithms
  and Architectures}, 2017, pp. 1--1.

\bibitem{rubinstein2019seth}
A.~Rubinstein and V.~V. Williams, ``Seth vs approximation,'' \emph{ACM SIGACT
  News}, vol.~50, no.~4, pp. 57--76, 2019.

\bibitem{abboud2015tight}
A.~Abboud, A.~Backurs, and V.~V. Williams, ``Tight hardness results for lcs and
  other sequence similarity measures,'' in \emph{2015 IEEE 56th Annual
  Symposium on Foundations of Computer Science}.\hskip 1em plus 0.5em minus
  0.4em\relax IEEE, 2015, pp. 59--78.

\bibitem{bringmann2021translating}
K.~Bringmann and A.~Nusser, ``Translating hausdorff is hard: Fine-grained lower
  bounds for hausdorff distance under translation,'' in \emph{37th
  International Symposium on Computational Geometry}, 2021.

\bibitem{abboud2018if}
A.~Abboud, A.~Backurs, and V.~V. Williams, ``If the current clique algorithms
  are optimal, so is valiant's parser,'' \emph{SIAM Journal on Computing},
  vol.~47, no.~6, pp. 2527--2555, 2018.

\bibitem{abboud2019subquadratic}
A.~Abboud, V.~Cohen-Addad, and H.~Houdroug{\'e}, ``Subquadratic
  high-dimensional hierarchical clustering,'' \emph{Advances in Neural
  Information Processing Systems}, vol.~32, 2019.

\bibitem{WILLIAMS2005}
R.~Williams, ``A new algorithm for optimal 2-constraint satisfaction and its
  implications,'' \emph{Theoretical Computer Science}, vol. 348, no.~2, pp.
  357--365, 2005.

\end{thebibliography}
